\documentclass[lettersize, journal]{IEEEtran}

\usepackage[utf8]{inputenc}
\usepackage{comment}

\usepackage[colorlinks=false]{hyperref}
\usepackage{orcidlink}   

\usepackage{amsmath,amsfonts,amssymb}

\usepackage{fancyhdr}

\usepackage{booktabs}
\usepackage{multirow}
\usepackage{array}
\usepackage{tabularx,makecell}
\newcolumntype{Y}{>{\centering\arraybackslash}X}
\newcolumntype{L}{>{\raggedright\arraybackslash}X}

\allowdisplaybreaks

\usepackage{graphicx}
\usepackage{float}
\usepackage{wrapfig}
\usepackage{placeins}
\usepackage{subcaption}
\usepackage{cite}

\usepackage{algorithm}
\usepackage{algorithmic}

\usepackage{amsmath, amssymb, amsfonts}
\usepackage{amsthm}

\newtheorem{theorem}{Theorem}
\newtheorem{assumption}{Assumption}
\newtheorem{lemma}{Lemma}
\newtheorem{proposition}{Proposition}
\newtheorem{corollary}{Corollary}

\theoremstyle{definition}
\newtheorem{definition}{Definition}

\allowdisplaybreaks

\usepackage{xcolor}
\newcommand{\com}[1]{}

\newcommand{\rev}[1]{{\color{black}#1}}

\newcommand{\reve}[1]{{\color{black}#1}}

\ifodd 0
\newcommand{\remove}[1]{{\color{black}#1}}
\newcommand{\add}[1]{}
\else
\newcommand{\remove}[1]{}
\newcommand{\add}[1]{{\color{black}#1}}
\fi

\usepackage{pifont}

\begin{document}

\title{Multi-User mmWave Beam and Rate Adaptation via Combinatorial Satisficing Bandits}

\author{\IEEEauthorblockN{Emre~\"Ozy{\i}ld{\i}r{\i}m\,\raisebox{0.7ex}{\orcidlink{0009-0003-5842-9782}}, Bar{\i}{\c s}~Yayc{\i}\,\raisebox{0.7ex}{\orcidlink{0009-0002-8488-540X}}, Umut~Eren~Akturk\,\raisebox{0.7ex}{\orcidlink{0009-0004-1094-9145}}, and Cem~Tekin\,\raisebox{0.7ex}{\orcidlink{0000-0003-4361-4021}},~{\em Senior Member, IEEE}}%
\thanks{This work was supported in part by the Scientific and Technological Research Council of Türkiye (TUBITAK) BILGEM Grant; TUBITAK Grant 124E065; by the Turkish Academy of Sciences Distinguished Young Scientist Award Program (TUBA-GEBIP-2023); by TUBITAK 2024 Incentive Award.}
\thanks{The authors are with the Department of Electrical and Electronics Engineering, Bilkent University, Ankara, T\"urkiye (e-mail: emre.ozyildirim@ug.bilkent.edu.tr; b.yayci@ug.bilkent.edu.tr; eren.akturk@bilkent.edu.tr; cemtekin@ee.bilkent.edu.tr). E. \"Ozy{\i}ld{\i}r{\i}m and B. Yayc{\i} contributed equally to this work.}%
\thanks{A preliminary version of this work was presented at the AI and ML for Next-Generation Wireless Communications and Networking Workshop at NeurIPS 2025 \cite{ozyldrm2025satisficing} (non-archival).}
}

\maketitle
\begin{abstract}
We study downlink beam and rate adaptation in a multi-user mmWave MISO system where multiple base stations (BSs), each using \rev{analog/digital hybrid codebook based beamsteering driven by a small number of RF chains,} serve multiple single-antenna user equipments (UEs) with one beam per UE and discrete data transmission rates. BSs learn about transmission success based on ACK/NACK feedback. To encode service goals, we introduce a satisficing \rev{QoS} threshold $\tau_r$ \rev{on the network-average per-UE throughput} and cast joint beam and rate adaptation as a combinatorial semi-bandit over beam-rate \rev{tuples, subject to per-BS RF-chain constraints.} Within this framework, we propose SAT-CTS, a lightweight,
threshold-aware policy that blends conservative confidence
estimates with posterior sampling. \reve{We analyze a reset variant that reinitializes the
posterior at each committed Thompson-sampling phase and
evaluate a retain variant that preserves it across phases.}
Our theoretical contribution provides the first finite-time regret bounds for combinatorial semi-bandits with satisficing objective: when $\tau_r$ is realizable, we upper bound the cumulative satisficing regret to the target with a time-independent constant, and when $\tau_r$ is non-realizable, we show an $O((\log T)^2)$ standard regret bound. \reve{Experiments show that SAT-CTS (retain) reduces mean
satisficing regret relative to CTS across all tested realizable
targets, while SAT-CTS (reset) does so at most tested realizable
targets. Both variants improve over CUCB throughout these
targets. The retain variant also maintains competitive
standard regret relative to CTS near and above the optimal
throughput threshold. Together with favorable average
throughput and fairness across users, these results indicate
that feedback-efficient learning can meet a network-level
QoS target without channel state knowledge.}
\end{abstract}


\section{Introduction}

Millimeter-wave (mmWave) communication has emerged as a cornerstone technology for next-generation wireless networks, driven by the relentless growth in mobile data demand and the proliferation of bandwidth-hungry applications such as interactive extended reality \cite{struye2022covrage} and ultra-high-definition streaming \cite{liu2018mec}. By leveraging large swaths of underutilized spectrum, mmWave enables unprecedented peak data rates and supports ultra-low latency services that are essential for such emerging applications. At the same time, the unique propagation characteristics of mmWave, such as high path loss, susceptibility to blockage, and strong directionality, introduce fundamental challenges in acquiring channel state information (CSI), efficiently aligning beams, and adjusting data rates to meet service demands \cite{rappaport2015wideband}.   

\rev{Accurate CSI is often costly to acquire in mmWave systems, since large antenna arrays require substantial pilot overhead and highly directional links entail extensive beam training \cite{deng2025csi}.} This motivates beam adaptation without CSI, where the transmitter avoids explicit channel estimation and instead uses minimal link-layer feedback \cite{mamba, Wu2019, ktari2024machine}. In particular, ACK/NACK feedback provides a reliable binary signal indicating whether the current beam configuration can sustain the required packet delivery (and thus throughput) under the prevailing propagation conditions. By iteratively updating the beam choice based solely on this feedback, the system can implicitly track the good beam directions and maintain throughput targets while eliminating the CSI acquisition burden.

We study downlink data transmission in a multi-user, multi-base station Multiple Input Single Output (MISO) system, where multiple base stations (BSs) serve single-antenna user equipments (UEs) \cite{1266912}. In mmWave communications, transmitters employ highly directional narrow beams to mitigate severe path loss and blockage \cite{HurTCOM2013}. Reliable links require these beams to be accurately steered toward the UEs \cite{Zhou2017EnhancedRA, wei2022fast}.
Each UE can be served by one of several candidate beams across different BSs. A BS transmits data to a UE using the selected beam at the highest feasible rate, determined by the chosen modulation and coding scheme (MCS). We model this as a joint beam and rate adaptation problem without CSI, where after each transmission the BS receives binary ACK/NACK feedback for the selected beam–rate pair.

\begin{figure}[t]
\centering
\includegraphics[width=0.85\columnwidth]{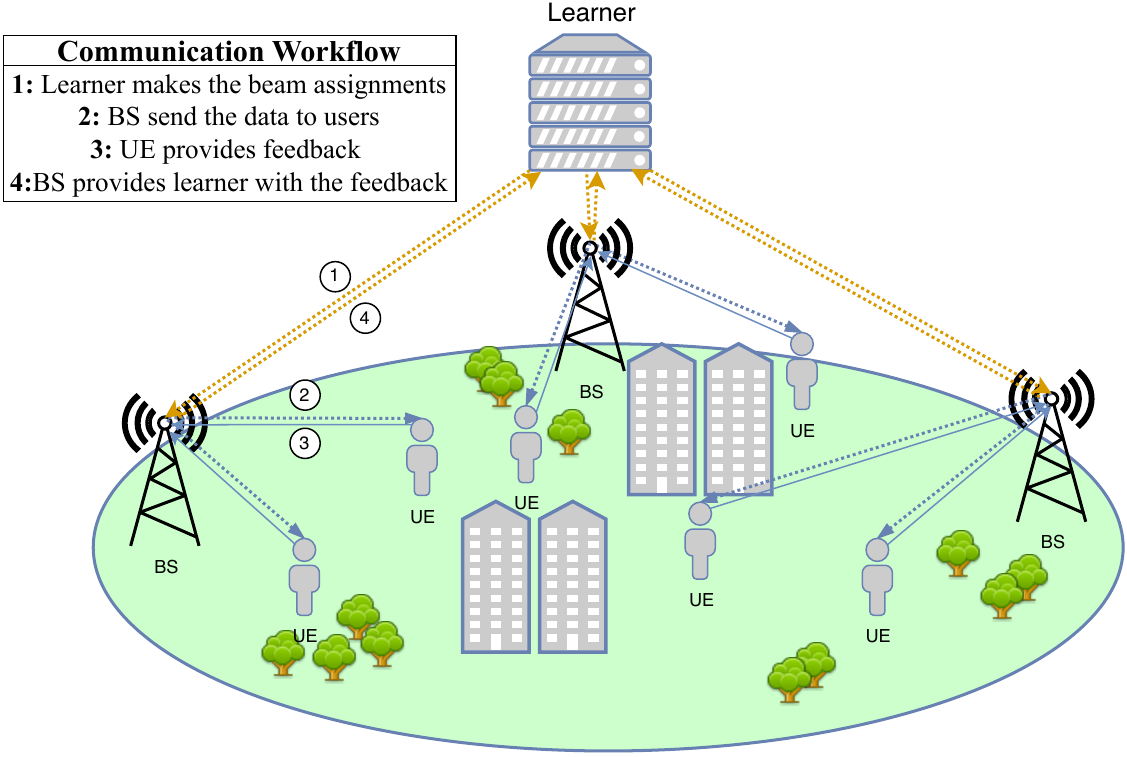}
\caption{System model.}
\label{fig:sys}
\end{figure}

We consider a centralized architecture in which a learner coordinates beam and rate assignments for the BSs and UEs at the beginning of each time slot over an ultra-low latency control channel (Fig. \ref{fig:sys}). At the end of each slot, BSs return feedback to the learner, which updates assignments to improve system performance. We cast the problem of learning the best beam-rate assignments by relying solely on ACK/NACK feedback as a sequential decision-making problem. This setup exhibits the canonical exploration–exploitation tradeoff: the learner must explore different beam-rate pairs to learn which ones perform well while exploiting known good configurations to satisfy throughput requirements. This tradeoff is optimality balanced by multi-armed bandit (MAB) algorithms \cite{lattimore2020bandit} that can adapt beams and rates in a completely data-driven manner.  

While a wide range of MAB algorithms has been proposed for identifying the best choice (i.e., the best {\em super-arm}) among combinatorially many configurations (i.e., {\em super-arms}) \cite{Kveton2015SemiBandit, pmlr-v28-chen13a, pmlr-v80-wang18a, huyuk2020thompson}, many of these methods are inherently designed to prioritize long-term optimality and thus may over-explore in practice rather than quickly settling on a configuration that already satisfies the throughput requirements, as a result, they can exhibit slow convergence and poor early performance, leading to suboptimal communication reliability and throughput during the early rounds of communication. 

Departing from the prior work, we focus on designing algorithms that learn a good-enough configuration in the most sample efficient way. We do this by modeling the joint beam and rate adaptation process as a satisficing combinatorial MAB (CMAB) with semi-bandit feedback, and develop a scalable learning algorithm, called {\em satisficing combinatorial Thompson sampling} (SAT-CTS), that ensures \rev{network-wide} average \rev{throughput while empirically exhibiting competitive fairness across users.} Our approach follows Herbert Simon’s bounded-rationality perspective, where agents aim to reach an aspiration level under limited time and information \cite{SimonQJE1955}. Instead of focusing on convergence to a maximizer, we measure how quickly the assignments achieve a satisficing threshold by analyzing a satisficing version of regret.

We provide finite-time regret bounds for SAT-CTS under two scenarios. When the satisficing threshold $\tau_r$ is \emph{realizable} (i.e., achievable by some assignment with margin), we prove a horizon-free upper bound on the expected cumulative satisficing regret that depends only on system parameters. When $\tau_r$ is \emph{non-realizable} (higher than the best achievable throughput), we show that  \rev{the expected
total number of slots outside committed combinatorial Thompson sampling (CTS) \cite{pmlr-v80-wang18a} rounds is finite.
Summing the regret contributions of the restarted CTS rounds
then yields an $O((\log T)^2)$ standard regret bound.} To the best of our knowledge, these are the first finite-time regret guarantees for CMAB with satisficing objectives.

We evaluate our approach in a simulated multi-BS, multi-UE MISO system with realistic channel vectors generated by the DeepMIMO simulator \cite{alkhateeb2019deepmimogenericdeeplearning}. \rev{Terminals are static, so the
large-scale geometry of each link is fixed over the horizon; the channel,
however, is not time-invariant, since moving people, vehicles, foliage, and
other surrounding objects act as reflectors and scatterers and produce
time-varying multipath even when both the transmitter and the receiver are
stationary \cite{rappaport2015mmwave, bas2018dynamic, zhao2015doppler}. We
capture this with a block fading model  \cite{8423195,
biglieri1998fading}.} \rev{We also evaluate SAT-CTS (retain), a variant that
preserves posterior information across committed CTS rounds.
This variant achieves lower mean satisficing regret than}
CTS \cite{pmlr-v80-wang18a} and combinatorial upper
confidence bound (CUCB) \cite{pmlr-v28-chen13a} baselines
\rev{across all tested realizable targets.
Our theoretical guarantees apply to SAT-CTS (reset),
which restarts each committed CTS round with fresh priors.} Compared with the preliminary workshop version \cite{ozyldrm2025satisficing}, the present paper uses a committed-round SAT-CTS design with reset fresh CTS rounds and removes the optimistic UCB gate, which enables the finite-time regret analysis.

In a nutshell, we propose the first scalable, binary-feedback-enabled beam and rate adaptation algorithm for mmWave systems in the absence of CSI. We rigorously prove that SAT-CTS achieves bounded expected cumulative satisficing regret under our realizability assumptions, and demonstrate \rev{the empirical satisficing-regret
benefits of its posterior retaining variant}.

The rest of the paper is organized as follows. Related work is given in Section \ref{sec:related}. The beam and rate adaptation problem is formalized in Section \ref{sec:problem}. SAT-CTS algorithm is proposed in Section \ref{sec:algorithms}. Regret analysis of SAT-CTS under realizable target is done in Section \ref{sec:realizable}. Experimental results on the DeepMIMO simulator are reported in Section \ref{sec:experiments}. Concluding remarks are given in Section \ref{sec:conclusion}.

\section{Related Work} \label{sec:related}

\newcommand{\cmark}{\ding{51}}  
\newcommand{\xmark}{\ding{55}} 
\begin{table}[!t]
\centering
\caption{\rev{Comparison with related works.}}
\label{tab:related-ieee}
\scriptsize
\setlength{\tabcolsep}{3pt}
\renewcommand{\arraystretch}{1.08}

\begin{tabularx}{\linewidth}{@{}l
  >{\centering\arraybackslash}X
  >{\centering\arraybackslash}X
  >{\centering\arraybackslash}X
  >{\centering\arraybackslash}X@{}}
\toprule
\textbf{Work} &
\textbf{Combinatorial Setup} &
\textbf{MAB \& Comm. System Together} &
\textbf{Satisficing Threshold} &
\textbf{Beam + Data (Align \& Rate) } \\
\midrule
MAMBA~\cite{mamba}            & \xmark & \cmark & \xmark & \cmark \\
HBA~\cite{Wu2019}     & \xmark & \cmark & \xmark & \xmark \\
CCBM~\cite{li2024contextualcombinatorialbeammanagement}
                               & \cmark & \cmark & \xmark & \xmark \\
CCVB for BS~\cite{9373701}    & \cmark & \cmark & \xmark & \xmark \\
FBA~\cite{fastbeam}           & \xmark & \xmark & \xmark & \xmark \\
CCV--MAB~\cite{pmlr-v108-nika20a}
                               & \cmark & \xmark & \xmark & \xmark \\
\midrule
\textbf{Our work}              & \cmark & \cmark & \cmark & \cmark \\
\bottomrule
\end{tabularx}
\end{table}

In mmWave networks, the central challenge is assigning beams together with appropriate rates so that data transmissions succeed under directional links and SNR thresholds. Several non-bandit approaches tackle the alignment phase. For instance, Hassanieh et al. \cite{fastbeam} design fast probing to identify a good beam, but do not address multi-user assignment or data rate selection. Marzi et al. \cite{marzi2016compressive} employ compressive sensing techniques that utilize the sparsity of the mmWave channel to reduce the alignment overhead. Wang et al. \cite{wang2009beam} propose a hierarchical multi-resolution search procedure to identify the optimal beam while minimizing the probing overhead. Several other works \rev{exploit auxiliary information to facilitate beam selection, including user} position \rev{\cite{ghatak2020beamwidth,ali2017millimeter}, while others rely on channel information that is assumed perfectly known \cite{Ayach_2014}, obtained through limited feedback \cite{Alkhateeb2015MU}, or inferred from lower-frequency measurements \cite{deng2025csi}}

There also exists plenty of MAB-based formulations for beam alignment. Wei et al. \cite{wei2022fast} treat fast beam alignment as a pure exploration problem, identifying the best beam from pilot measurements (received signal strength) without sending real data or choosing rates. Ghosh et al. \cite{Ghosh2024UB3} propose a MAB algorithm that utilizes the unimodal structure of the RSS to maximize the probability of the identification of the best beam given a fixed transmission budget. Several other works, such as \cite{hashemi2018efficient, Wu2019, mamba,zhang2021mmwave}, approach the beam alignment problem from a regret minimization perspective and propose algorithms that continuously explore and exploit. However, continuous exploration can lead to suboptimality if the goal is to identify the best beam for data transmission during an alignment phase with pilot symbols.

On the other hand, when used to directly select beams for data transmission, continuous exploration can result in suboptimal beam selections for data transfer. In contrast to these works, we apply MAB algorithms for beam tracking without requiring pilot based beam alignment. Instead, we aim directly for data transmission in each slot with the chosen beam (see Fig. \ref{fig:bsue} for the process of beam selection and data transmission). Moreover, to reduce suboptimal explorations, we incorporate a satisficing approach where the goal is to select beam-rate pairs that guarantee a target level of throughput demand instead of the highest throughput.  

Along regret minimization direction, contextual combinatorial beam management \cite{li2024contextualcombinatorialbeammanagement} extends alignment to a contextual, combinatorial setting with multiple UEs and BSs and a probing budget, and tries to maximize a contextual reward. Beam alignment in non-stationary environments is considered in \cite{min2024beam} and \cite{ghatak2024best}. Another related work \cite{9373701} focuses on the problem of multi-user small BS association in a contextual setup with dynamic UE presence using contextual combinatorial volatile MAB without beam management. Relatedly, CCV-MAB \cite{pmlr-v108-nika20a} addresses contextual combinatorial bandits with time-varying availability, offering regret guarantees via adaptive discretization. Yet it remains communication-agnostic (no beam/rate modeling or ACK/NACK signals). A procedure similar to adaptive discretization, called dynamic beam zooming, is proposed in \cite{blinn2025mab} for beam alignment and tracking. MAMBA \cite{mamba} jointly adapts beam and MCS from ACK/NACK, yet remains a single user, single BS formulation without combinatorial assignment. \rev{Works that incorporate
explicit rate requirements, such as the joint beam allocation and scheduling
formulation of \cite{Gopalam2024}, assume knowledge of the per user achievable rates and beam capacity structure for offline optimization, rather than learning the assignment from binary ACK/NACK feedback}

\rev{To clarify the distinction in learning objectives, satisficing
bandits \cite{SimonQJE1955,feng2025satisficing} seek an aspiration level
rather than the reward-maximizing arm, but existing formulations typically
consider arm selection without combinatorial assignment structure. We
extend this principle to a combinatorial semi-bandit with per-arm ACK/NACK
feedback and one-beam-per-UE and RF-chain assignment constraints.
\emph{QoS-aware and constrained bandits}, including bandits with knapsacks
\cite{badanidiyuru2018bandits}, constrained Thompson sampling
\cite{saxena2020thompson}, network-guarantee bandits
\cite{cai2023learning}, and combinatorial extensions
\cite{liu2022combinatorial,wu2025low}, retain a reward maximization objective
while separately accounting for resource, reliability, or per arm
performance constraints. Such methods generally require additional
constraint-management mechanisms.}

\rev{Taken together, prior methods differ from SAT-CTS in at least one of the
feedback model, learning objective, assignment structure, or system setting.
To the best of our knowledge, SAT-CTS is the first framework to combine
multi-BS and multi-UE combinatorial assignment, CSI-free ACK/NACK learning,
a network-level satisficing throughput target, and joint beam and rate
adaptation, with finite-time guarantees for both realizable and
non-realizable targets, as summarized in Table~\ref{tab:related-ieee}.}

Our research fills the gap between alignment-only pilot methods and single-link bandits by providing a target-aware, multi-user assignment mechanism that operates directly on the data plane with provable performance guarantees.

\section{Problem Formulation}
\label{sec:problem}

We consider a multi-user mmWave MISO system in which $B$ BSs, each with $N$
transmit antennas, serve $M$ single-antenna UEs as in Fig.~\ref{fig:sys}.
System parameters are summarized in Table~\ref{tab:notation_multi}. For an
integer $n$, let $[n]:=\{1,\ldots,n\}$.

\begin{table}[t]
  \centering
  \small
  \caption{\rev{Summary of notation.}}
  \label{tab:notation_multi}
  \begin{tabular*}{\columnwidth}{@{\extracolsep{\fill}}ll@{}}
    \toprule
    Symbol & Description \\
    \midrule
    \multicolumn{2}{@{}l}{\textbf{System and transmitter}} \\
    $B$, $M$ & Number of BSs, number of UEs \\
    $N$, $K$ & Antennas per BS, codewords per codebook \\
    $N_{\mathrm{RF},b}$ & RF chains at BS $b$ \\
    $T$ & Time horizon \\
    $\mathcal{C}_b$, $\mathbf{f}_{b,k}$ & Codebook of BS $b$, its $k$th codeword \\
    $\mathcal{U}_b(t)$ & UEs served on the active RF chains of BS $b$ 
    \\
     $\mathcal{K}_b(t)$ & Codewords loaded on the active RF chains of BS $b$
    \\
    $\mathbf{F}_{\mathrm{RF},b}$, $\mathbf{F}_{\mathrm{BB},b}$ & Analog, digital precoding stage \\
    $p_b$ & Transmit power per active beam \\
    $\lambda$, $d$ & Wavelength, antenna spacing \\
    \addlinespace
    \multicolumn{2}{@{}l}{\textbf{Channel and fading}} \\
    $\mathbf{h}_{m,b}(t)$ & Channel from BS $b$ to UE $m$ in slot $t$ \eqref{eqn:channelgain} \\
    $L_{m,b}$ & Number of paths \\
    $\beta_{m,b,\ell}(t)$, $\theta_{m,b,\ell}$ & Complex gain of path $\ell$ in slot $t$, its AoD \\
    $P_{m,b,\ell}$ & Average power of path $\ell$ \\
    $\mathbf{a}(\cdot)$ & Array steering vector \\
    \addlinespace
    \multicolumn{2}{@{}l}{\textbf{Link and block error}} \\  
    $\sigma_m^2$ & Noise power \\
    $\gamma_{m,(b,k)}(t)$ & Received SNR \eqref{eqn:snr_def} \\
    $\mathcal{R}$, $r$ & MCS rate set, rate (b/channel use) \\
    $\mathrm{BLER}_r(\cdot)$ & Block-error curve of MCS $r$ \\
    $x_{m,t}$ & ACK/NACK indicator of UE $m$ \eqref{eqn:ack_model} \\
    $\psi_{m,(b,k),r}$ & Success probability \eqref{eqn:psidef} \\
    \addlinespace
    \multicolumn{2}{@{}l}{\textbf{Bandit formulation}} \\
    $\mathcal{A}$, $i$ & Base-arm set, base arm $(m,(b,k),r)$ \\
    $\mathcal{S}$, $s$ & Super-arm set, super arm \\
    $\pi_t$, $\rho_t$ & Beam, rate assignment at slot $t$ \\
    $g(s)$ & Average expected throughput of $s$ \\
    $\tau_r$ & Satisficing threshold \\
    $\Delta_{\mathrm{sat}}(\cdot)$ & Per-round satisficing regret \\
    $R^{S}(T)$, $R^{\mathrm{std}}(T)$ & Cumulative satisficing, standard regret \\
    \bottomrule
  \end{tabular*}
\end{table}

\subsection{\rev{Transmitter Architecture and Beamforming Codebook}}
\label{subsec:architecture}

\rev{Each BS is equipped with $N$ antennas driven by a much smaller number of
radio-frequency (RF) chains, $N_{\mathrm{RF},b}\ll N$, which is the standard
hybrid analog/digital front end at mmWave frequencies, where a dedicated RF
chain per antenna is prohibitive in cost and power
consumption~\cite{Ayach_2014,Molisch2017}. The downlink precoder of BS~$b$ in
slot~$t$ accordingly factorizes into an analog and a digital stage,}
\rev{\begin{align}
\mathbf{F}_b(t)=\mathbf{F}_{\mathrm{RF},b}(t)\,\mathbf{F}_{\mathrm{BB},b}(t),
\label{eqn:hybrid_factorization}
\end{align}
with $\mathbf F_{\mathrm{RF},b}(t)
\in\mathbb C^{N\times|\mathcal U_b(t)|}$
and $\mathbf F_{\mathrm{BB},b}(t)
\in\mathbb C^{|\mathcal U_b(t)|\times|\mathcal U_b(t)|}$,
where $\mathcal U_b(t)$ denotes the set of UEs served by
BS~$b$ in slot~$t$. Only active RF chains are represented
in these matrices.}
The analog stage is implemented with phase shifters and is
therefore restricted to constant-modulus entries; as in codebook-based hybrid
designs~\cite{Codebook2017,Sohrabi2016}, each of its columns is drawn from the
predefined codebook of size $K$,
\begin{align}
\mathcal{C}_b:=\bigl\{\mathbf{f}_{b,k}\in\mathbb{C}^{N}:k\in[K]\bigr\},
\qquad \|\mathbf{f}_{b,k}\|_2=1,
\label{eqn:codebook}
\end{align}
\rev{We assume $N_{\mathrm{RF},b}\le K$ for every $b\in[B]$,} and BS~$b$ activates one codeword per active RF chain,
\begin{align}
\mathbf{F}_{\mathrm{RF},b}(t)=\bigl[\mathbf{f}_{b,k}\bigr]_{k\in\mathcal{K}_b(t)},
\qquad |\mathcal{K}_b(t)| = |\mathcal{U}_b(t)| \le N_{\mathrm{RF},b}
\label{eqn:analog_stage}
\end{align}
\rev{where $\mathcal K_b(t)$ is the ordered list of codeword
indices assigned to the active RF chains in slot~$t$,
with repetitions allowed and one entry per served UE.} A purely analog front end ($N_{\mathrm{RF},b}=1$) can form only one
analog beam direction at a time; \rev{since each served UE occupies one RF
chain, this restricts the BS to a single served UE per slot.} The
simultaneous multi-UE service assumed in our assignment is therefore provided by
the multiple RF chains, and the number of UEs a BS can serve in one slot is
bounded by $N_{\mathrm{RF},b}$, as in multi-user hybrid
precoding~\cite{Alkhateeb2015MU}.

Because the analog stage is drawn from a fixed codebook, the transmitter is
configured purely by \emph{selecting} codewords rather than by computing
channel-dependent precoders. The digital stage $\mathbf{F}_{\mathrm{BB},b}(t)$
is not used for channel-dependent processing; it reduces to a selection matrix
that maps the data stream of each served UE to one RF chain. Each BS thus
operates in the analog beamsteering (beam-selection) mode used as a benchmark in
multi-user hybrid designs~\cite{Alkhateeb2015MU}, and its transmission is
determined entirely by the list of activated codewords $\mathcal{K}_b(t)$.


\subsection{\rev{Channel Model}}
\label{subsec:channel}

\rev{We consider static terminals: the BS and UE locations and the large scale
propagation geometry of each link remain fixed over the learning horizon. However, static terminals do not necessarily imply a time-invariant channel. MmWave propagation studies show that moving people, vehicles, foliage,
and other surrounding objects may act as reflectors and scatterers, producing
time-varying multipath even when both the transmitter and receiver are
stationary~\cite{rappaport2015mmwave,bas2018dynamic,zhao2015doppler}.
We model this time variation using a block-fading abstraction. The
channel remains constant during one decision round, which carries one transport
block, and takes an independent random realization in the next round. This is
consistent with the conventional coherent block-fading model used in
finite-blocklength analysis ~\cite{8423195,
biglieri1998fading}.}

The mmWave channel is represented using the sparse geometric
Saleh--Valenzuela model~\cite{SalehValenzuela1987}. In decision round~$t$, the
channel from BS~$b$ to UE~$m$ is
\begin{align}
\mathbf h_{m,b}(t)
=
\sqrt{N}
\sum_{\ell=1}^{L_{m,b}}
\beta_{m,b,\ell}(t)\,
\mathbf a\bigl(\cos\theta_{m,b,\ell}\bigr),
\label{eqn:channelgain}
\end{align}
where $\theta_{m,b,\ell}$ is the nominal AoD of path~$\ell$ and
$\beta_{m,b,\ell}(t)$ is its complex path coefficient. \rev{The path count and
nominal AoDs remain fixed because the large-scale geometry is static, whereas
the complex path coefficients vary across decision rounds due to unresolved
small-scale environmental changes \cite{Ayach_2014}. Consequently, $\mathbf h_{m,b}(t)$ is a
random channel vector whose statistical characteristics remain unchanged over
the learning horizon.}

The unit-norm steering vector of an $N$-element uniform linear array with
inter-element spacing $d$ and carrier wavelength $\lambda$ is
\begin{align}
\mathbf a(\cos\theta)
=
\frac{1}{\sqrt N}
\left[
1,
e^{j\frac{2\pi}{\lambda}d\cos\theta},
\ldots,
e^{j\frac{2\pi}{\lambda}d(N-1)\cos\theta}
\right]^T,
\label{eqn:steering}
\end{align}
as in~\cite{Ghosh2024UB3}.

\begin{assumption}[\rev{Static geometry and stationary block fading}]
\label{ass:stationary}
\rev{The large-scale geometry of every link remains fixed over the learning horizon,
and the small-scale fading is independent across rounds. Hence the distribution
of $\mathbf h_{m,b}(t)$, and the received SNR, does not depend
on~$t$.}
\end{assumption}


\subsection{\rev{Received Signal and SNR}}
\label{subsec:snr}

\rev{Following the orthogonal resource-block allocation adopted in combinatorial
multi-user association at mmWave~\cite{9373701}, each BS~$b$ partitions its
spectrum into $q_b\ge N_{\mathrm{RF},b}$ orthogonal sub-channels, realized by
frequency division or two-dimensional time--frequency resource blocks in which
each sub-channel is assigned to at most one
link~\cite{Sokun2017,BistritzLeshem2019}, and assigns each served UE exactly one distinct sub-channel of bandwidth $B_{\mathrm{sc}}$. Unused sub-channels remain idle. Neighbouring cells are separated by a dedicated-band plan
as in~\cite{9373701}.}

Let UE~$m$ be served by BS~$b$ through codeword $\mathbf{f}_{b,k}$ in slot~$t$,
and let $s_m$ denote its unit-power data symbol. On the resource demodulated by
UE~$m$,
\begin{align}
y_m(t)=\sqrt{p_b}\,\mathbf{h}_{m,b}^{H}(t)\,\mathbf{f}_{b,k}\,s_m+n_m,
\label{eqn:received_signal}
\end{align}
where $n_m\sim\mathcal{CN}(0,\sigma_m^2)$ is AWGN. \rev{Here, $p_b$ is the fixed transmit power per active beam,
independent of the number of UEs served by BS~$b$. Equation
\eqref{eqn:received_signal} follows from the orthogonality of the resource
blocks together with the dedicated-band inter-cell allocation: the other beams
co-scheduled by BS~$b$, and the beams of neighbouring cells, occupy disjoint
resources and therefore contribute nothing to the resource UE~$m$ demodulates.
Had the allocation not been orthogonal, the received signal would additionally
carry intra-cell leakage from the remaining beams of the serving BS and
inter-cell interference, and the relevant quantity would be an SINR. Under the
orthogonal allocation the interference power vanishes on UE~$m$'s resource, the
SINR coincides with the received SNR, and the analysis is carried out in terms
of}
\begin{align}
\gamma_{m,(b,k)}(t)
=\frac{p_b\,\bigl|\mathbf{h}_{m,b}^{H}(t)\,\mathbf{f}_{b,k}\bigr|^{2}}
      {\sigma_m^{2}}.
\label{eqn:snr_def}
\end{align}

\rev{The noise power at UE $m$ is $\sigma_m^2 = \eta_0B_{\mathrm{sc}}\,F_m$,
where $\eta_0$ is the thermal noise power spectral density, $B_{\mathrm{sc}}$ is
the per-UE sub-channel bandwidth, and $F_m$ is the noise figure of UE $m$'s
receiver. In our experiments all UEs use identical receivers, so
$F_m = F$ for all $m\in[M]$.}

\subsection{\rev{Rate Set and Block-Error Model}}
\label{subsec:bler}

Transmissions use a finite set of modulation and coding
schemes (MCSs), $\mathcal R=\{r_1<\cdots<r_R\}$, where
each $r\in\mathcal R$ denotes the
\rev{nominal spectral efficiency of an MCS.
We use these values as rate weights; throughout the paper,
throughput denotes the ACK weighted nominal rate rather
than exact transport block payload goodput.}

\rev{We assume that each scheduled block returns exactly one ACK/NACK bit
indicating its decoding outcome. HARQ retransmissions and
soft combining~\cite{3gpp38214} are not modeled.
All transmissions use a fixed number of physical resource blocks (PRBs), orthogonal frequency-division multiplexing (OFDM) symbols, and layers, and a fixed transport-block-size construction, so the
transport block spans a fixed number of channel uses and, for each MCS, the
decoding outcome depends on the received SNR alone.}

\rev{For a coded transmission at a given MCS, the block-error rate (BLER) is the
probability that the receiver fails to decode the transport block. For a fixed
MCS the BLER decreases from near one to near zero as the received SNR increases;
because the block is of finite length this transition occupies a nonzero SNR
range rather than a step, an effect characterized by finite-blocklength
information theory~\cite{polyanskiy2010channel}.}

\rev{Rather than adopting an analytical approximation of this curve, we obtain the
BLER versus SNR characteristic of each MCS from
link-level simulation using the NVIDIA Sionna
framework~\cite{sionna}, which implements the 5G NR physical layer, including
LDPC coding and rate matching, QAM modulation, and channel decoding, in
accordance with the 3GPP specifications. These curves are
computed offline once and used as the per-MCS success model.}

Let $x_{m,t}$ denote the ACK indicator when UE~$m$ is served through beam
$(b,k)$ at rate $r$ in slot~$t$. \rev{Given the time slot's channel realization, the
transport block is successful with probability
$1-\mathrm{BLER}_{r}(\gamma_{m,(b,k)}(t))$. Averaging over the fading
distribution gives the mean success probability of arm $(m,(b,k),r)$,}
\begin{align}
\psi_{m,(b,k),r}
:=
\mathbb E_{\mathbf h}\bigl[1-\mathrm{BLER}_{r}(\gamma_{m,(b,k)}(t))\bigr],
\label{eqn:psidef}
\end{align}
\rev{which by Assumption~\ref{ass:stationary} does not depend on~$t$. The
ACK/NACK observation is therefore}
\begin{align}
x_{m,t}\sim\operatorname{Bernoulli}\bigl(\psi_{m,(b,k),r}\bigr),
\label{eqn:ack_model}
\end{align}
\rev{and the observations obtained whenever the arm is selected are
i.i.d.\ across rounds.} These success probabilities are unknown to the learner
and are estimated solely from ACK/NACK feedback.
\begin{figure}[t]
\centering
\includegraphics[width=0.95\columnwidth]{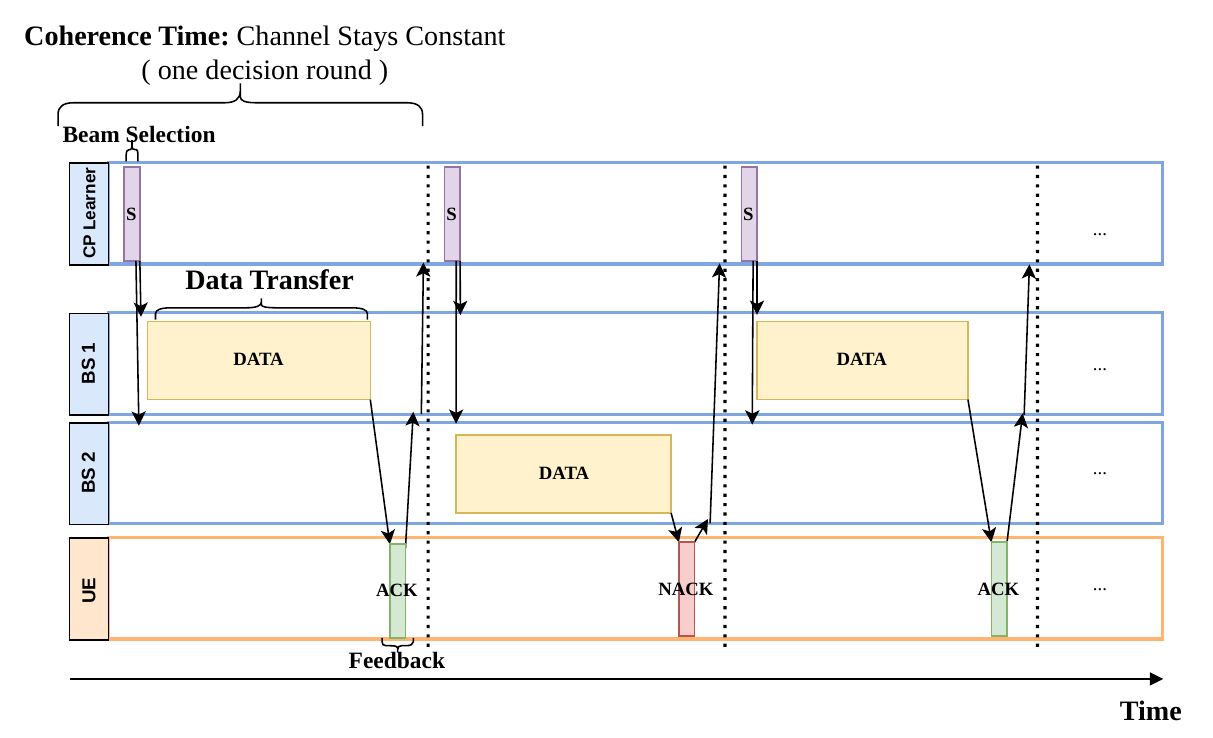}
\caption{\rev{Beam and rate assignment process.}}
\label{fig:bsue}
\end{figure}
\subsection{Combinatorial Multi-Armed Bandit Formulation}

Beam $k$ of BS $b$ is denoted by tuple $(b,k)$. We denote the set of all beams
by $\mathcal K := [B]\times[K]$. Additionally, we define a discrete set of
feasible transmission rates $\mathcal{R} = \{r_1, r_2, \ldots, r_R\}$ where
$r_1 < r_2 < \ldots < r_R$ represent the available data rates in bits per
channel use. Different data rates can be achieved by choosing a different MCS.
We assume $\sum_{b\in[B]} N_{\mathrm{RF},b}\geq M$.
Let $T$ represent the time horizon. We consider a centralized system where the
learner coordinates beam and rate selection for downlink transmission.

As presented in Fig.~\ref{fig:bsue}, the following events take place
sequentially at each time slot $t$. At the beginning of each time slot $t \in
[T]$, the learner assigns exactly one beam to each UE via a mapping
$\pi_t:\{1,\ldots,M\}\to\mathcal K$, and selects a transmission rate for each UE
via a mapping $\rho_t:\{1,\ldots,M\}\to\mathcal R$. Here, $\pi_{m,t} =
(b_{m,t},k_{m,t})$ represents the beam assigned to UE $m$ with $b_{m,t}$
representing the assigned BS and $k_{m,t}$ representing the beam of the assigned
BS. For each UE $m \in [M]$, the learner communicates this information with the
associated BS through an ultra-low latency control channel. After this
communication, BS $b_{m,t}$ uses its beam $\pi_{m,t} = (b_{m,t},k_{m,t})$ to
transmit data to UE $m$ at rate $\rho_{m,t}$.

A base arm is indexed by
$i=(m,(b,k),r)$, where $m\in[M]$ is the UE, $(b,k)\in\mathcal K$ is the
BS--beam pair, and $r\in\mathcal R$ is the rate. Let ${\cal A}$ represent the
set of base arms and let
\begin{align*}
{\cal S} :=  \{ & \{(\pi_1, \rho_1), \ldots, (\pi_M, \rho_M)\}
: \pi_m \in {\cal K},\ \rho_m \in \mathcal{R}, \\
& \rev{|\{m : b_m = b\}| \leq N_{\mathrm{RF},b}\ \ \forall b \in [B]} \}
\end{align*}
represent the set of super arms, \rev{where $b_m$ denotes the BS component of
$\pi_m$ and the last condition reflects that BS $b$ can serve at most
$N_{\mathrm{RF},b}$ UEs in a slot, one per RF chain}. For a super arm $s \in {\cal S}$, let
$\pi_m(s)$ be the beam assigned to UE $m$ and $\rho_m(s)$ be the rate selected
for UE $m$.

The instantaneous reward for UE $m$ at time $t$ is $r_{m,t} = \rho_{m,t} \times
x_{m,t}$, where the UE receives the selected rate $\rho_{m,t}$ if transmission
is successful and zero otherwise. The success probabilities
\rev{$\psi_{m,(b,k),r}$ in \eqref{eqn:psidef}} are unknown, which makes the
decision-making process without channel state knowledge challenging. An optimal
super arm is a super arm that maximizes the expected total throughput, which is
given as
\begin{align*}
    (\pi^*, \rho^*) &\in \arg\max_{(\pi,\rho) \in {\cal S}} \mathbb{E} \left[ \sum_{m=1}^M \rho_m \times x_{m} \right] \\
    &= \arg\max_{(\pi,\rho) \in {\cal S}} \sum_{m=1}^M \rho_m \times \psi_{m,\pi_m,\rho_m} ,
\end{align*}
where $\rho_m$ is the transmission rate selected for UE $m$, $x_{m} \in
\{0,1\}$ represents the random ACK/NACK feedback under super arm choice
$(\pi,\rho)$, and $\psi_{m,\pi_m,\rho_m}$ is the transmission success
probability for UE $m$ with beam $\pi_m$ and rate $\rho_m$.

Given assignment $S_t = \{(\pi_{1,t}, \rho_{1,t}),\dots,(\pi_{M,t},
\rho_{M,t})\}$, the learner observes at the end of slot $t$ the per-UE ACK/NACK
feedback $x_{m,t} \in \{0,1\}$ for $m=1,\dots,M$. This information, which will
be used by the learner for subsequent super arm selections, is known as the
semi-bandit feedback and is communicated by the BSs to the learner via the
ultra-low latency control link.

Define the \emph{satisficing threshold} as $\tau_r \in \mathbb{R}_+$
representing the target average throughput per UE. This threshold can be
interpreted as the desired minimum average data rate that should be achieved
across all UEs. The \emph{per-round satisficing regret} of $S_t$ is
\begin{align*}
    \Delta_{\mathrm{sat}}(S_t) := \left[\tau_r - \frac{1}{M}\sum_{m=1}^M \rho_{m,t} \cdot \psi_{m,\pi_{m,t},\rho_{m,t}} \right]_+ ,
\end{align*}
and the \emph{cumulative satisficing regret} over $T$ time slots is
\begin{align}
    R^S(T) := \sum_{t=1}^T \Delta_{\mathrm{sat}}(S_t) . \label{eqn:satisficingregret}
\end{align}

In the next section, we will propose an algorithm that minimizes the
satisficing regret. This way, the algorithm will achieve an average throughput
that is above $\tau_r$ in all but a finite number of initial rounds. A stricter
notion of regret, called {\em the standard regret}, compares the performance of
the learner with that of the best super arm. The $T$ round standard regret is
given as
\begin{align*}
    R_{\text{std}}(T)
    := \frac{1}{M} \sum_{t=1}^T  \sum_{m=1}^M ( \rho^*_{m}  \psi_{m,\pi^*_{m},\rho^*_{m}} - \rho_{m,t} \psi_{m,\pi_{m,t},\rho_{m,t}} ) .
\end{align*}
When $\tau_r$ is not realizable, i.e., there is no super arm that achieves
$\tau_r$, we will bound the standard regret.

\section{Fast Beam and Rate Adaptation via SAT-CTS}
\label{sec:algorithms}

\subsection{Algorithm Description}

\begin{algorithm}[!t]
\caption{SAT-CTS}
\label{alg:lcb-cts}
\begin{algorithmic}[1]
\REQUIRE $M$ UEs; beam set $\mathcal{K}=[B]\times[K]$;
  \rev{RF-chain counts $\{N_{\mathrm{RF},b}\}_{b=1}^{B}$;}
  rate set $\mathcal{R}=\{r_1<\cdots<r_R\}$;
  target $\tau_r$
\STATE For all $m\!\in\![M],\
  (b,k)\!\in\!\mathcal K,\ r\!\in\!\mathcal R$:
\STATE \quad $n_{m,(b,k),r}\!\leftarrow\!0,\
  s_{m,(b,k),r}\!\leftarrow\!0$
\STATE $\mathcal{S} \!=\! \bigl\{\{(\pi_1,\rho_1),\ldots,(\pi_M,\rho_M)\}:$
\STATE \quad $\pi_m=(b_m,k_m)\!\in\!\mathcal K,\;
\rho_m\!\in\!\mathcal R,$
\STATE \quad $\rev{\left|\{m\in[M]:b_m=b\}\right|
\le N_{\mathrm{RF},b}\;\forall\,b\!\in\![B]}\bigr\}$
\STATE \textbf{Primitives:}
\STATE \quad $\mathrm{BestAssign}(\mathrm{Score})$
\STATE \qquad $:= \arg\max_{s\in\mathcal S}\!\sum_{m=1}^M \mathrm{Score}_{m,\pi_m(s),\rho_m(s)}$
\STATE \quad $\mathrm{Avg}(\mathrm{Score},s)$
\STATE \qquad $:=\frac{1}{M}\sum_{m=1}^M \mathrm{Score}_{m,\pi_m(s),\rho_m(s)}$
\STATE CTS round counter $i\leftarrow 1$
\STATE $t\leftarrow 1$
\WHILE{$t\le T$}
  \STATE \textbf{// Compute indices}
  \FOR{each $m,(b,k),r$}
    \IF{$\rev{n_{m,(b,k),r}=0}$}
      \STATE \rev{$\mathrm{LCB}_{m,(b,k),r}\!\leftarrow\!0$;\
        $\mathrm{MEAN}_{m,(b,k),r}\!\leftarrow\!0$}
    \ELSE
      \STATE $\hat\psi\!\leftarrow\!
        s_{m,(b,k),r}/n_{m,(b,k),r}$\\ \
        $c\!\leftarrow\!\sqrt{3
        \log t/(2n_{m,(b,k),r})}$
      \STATE $\mathrm{LCB}_{m,(b,k),r}\!\leftarrow\!
        r\max\{0,\hat\psi\!-\!c\}$
      \STATE $\mathrm{MEAN}_{m,(b,k),r}\!\leftarrow\!
        r\hat\psi$
    \ENDIF
  \ENDFOR
  \STATE $S_L\!\leftarrow\!
    \mathrm{BestAssign}(\mathrm{LCB})$,\
    $S_M\!\leftarrow\!
    \mathrm{BestAssign}(\mathrm{MEAN})$
  \STATE \textbf{// Gate}
  \IF{$\mathrm{Avg}(\mathrm{LCB},S_L)\ge\tau_r$}
    \STATE Play $S_L$; observe and update
      shared counters; $t\!\leftarrow\!t\!+\!1$
  \ELSIF{$\mathrm{Avg}(\mathrm{MEAN},S_M)\ge\tau_r$}
    \STATE Play $S_M$; observe and update
      shared counters; $t\!\leftarrow\!t\!+\!1$
  \ELSE
    \STATE \textbf{// Committed CTS phase $i$}
    \STATE Reset: $A_{m,(b,k),r}\!\leftarrow\!1,\
      B_{m,(b,k),r}\!\leftarrow\!1\
      \ \forall\, m,(b,k),r$
    \FOR{$j=1$ to $\min(2^i,\,T\!-\!t\!+\!1)$}
      \FOR{each $m,(b,k),r$}
        \STATE $\tilde\psi_{m,(b,k),r}\!\sim\!
          \mathrm{Beta}(A_{m,(b,k),r},
          B_{m,(b,k),r})$
      \ENDFOR
      \STATE Play $\mathrm{BestAssign}(r\tilde\psi)$;
        observe $x_{m,t}\!\in\!\{0,1\}$
      \FOR{each $m$}
        \STATE $n_{m,\pi_{m,t},\rho_{m,t}}
          \mathrel{+}= 1$;\
          $s_{m,\pi_{m,t},\rho_{m,t}}
          \mathrel{+}= x_{m,t}$
        \STATE $A_{m,\pi_{m,t},\rho_{m,t}}
          \mathrel{+}= x_{m,t}$;\
          $B_{m,\pi_{m,t},\rho_{m,t}}
          \mathrel{+}= 1\!-\!x_{m,t}$
      \ENDFOR
      \STATE $t\leftarrow t+1$
    \ENDFOR
    \STATE $i\leftarrow i+1$
  \ENDIF
\ENDWHILE
\end{algorithmic}
\end{algorithm}

{\em Satisficing Combinatorial Thompson Sampling}
(SAT-CTS) selects BS-beam-rate combinations for each
UE to meet the target throughput. Its pseudocode is
given in Algorithm~\ref{alg:lcb-cts}.
Each arm $i=(m,(b,k),r)$ maintains shared counters
$(n_i,s_i)$ updated every round, where $n_i$ represents the number of times arm $i$ has been selected, and $s_i$ represents the number of times its transmission is successful.\footnote{To emphasize the dependence on $t$, the values of these counters are sometimes represented by $n_i(t)$ and $s_i(t)$.} The algorithm also keeps Beta priors that represent the distribution of transmission success probability of each arm. At each time step,
the algorithm computes empirical means
$\hat\psi_i(t)=s_i/n_i$ and confidence radius
$c(t,n_i(t))=\sqrt{\frac{3 \log t}{2n_i(t)}}$,
yielding $\mathrm{LCB}$ and $\mathrm{MEAN}$ indices. \rev{Both indices
must be defined for arms that have not yet been played. We set
$c(t,0):=+\infty$ and $\hat\psi_i(t):=0$ when $n_i(t)=0$, so that
$\mathrm{LCB}_i(t)=\mathrm{MEAN}_i(t)=0$ for every unplayed arm~\cite{jin2021pessimism}.} These are fed to the assignment oracle, and the gate plays the first candidate whose average
meets $\tau_r$ in the order LCB, MEAN.

If neither gate fires, the algorithm enters a
\emph{committed CTS phase}: it resets all Beta
priors to $\mathrm{Beta}(1,1)$ and runs CTS for
exactly $2^i$ rounds without rechecking the gate
(lines~31--44). During these rounds, both shared
counters and private Beta posteriors are updated.
After phase~$i$ completes, the phase counter
increments and the gate is re-evaluated. This
committed-phase structure ensures each CTS
instance runs on a deterministic horizon with
fresh priors, enabling a clean regret analysis
via a geometric-series framework \cite{feng2025satisficing}.

\subsection{Time Complexity Analysis}

\rev{For each assignment-oracle call, the score is first maximized over the
$K$ beams and $R$ rate levels for every UE--BS pair, requiring $O(MBKR)$
operations. This reduces the assignment problem to an $M\times B$ score
matrix. Each UE then independently selects its best BS, beam, and rate. If all
per-BS RF-chain constraints are satisfied, this selection is also optimal for
the capacitated problem, since the capacitated feasible set is a subset of the
unconstrained feasible set. If any RF-chain constraint is violated, we solve
the capacitated assignment problem using the Hungarian
algorithm~\cite{hungariankuhn} by expanding each BS $b$ into
$N_{\mathrm{RF},b}$ capacity slots, each of which can serve one UE, and
assigning UEs to slots on the resulting $M\times\sum_{b}N_{\mathrm{RF},b}$
score matrix; the beam and rate of each UE are then recovered from the stored
maximizer of its assigned BS. Since $\sum_b N_{\mathrm{RF},b}\ge M$, this step
requires $O\bigl(M^2\sum_b N_{\mathrm{RF},b}\bigr)$ operations. Consequently,
rounds in which the independent selection satisfies the RF-chain constraints
require $O(MBKR)$ operations, and rounds in which the RF-chain constraints are
active require $O\bigl(MBKR+M^2\sum_b N_{\mathrm{RF},b}\bigr)$ operations.}

\subsection{Main Regret Bounds and Their Discussion}

\label{sec:main_guarantees}

We next summarize the main finite-time guarantees of SAT-CTS. Before providing the main results, we introduce notation, definitions, and assumptions under which the main results hold.

\subsubsection{Notation}

\paragraph{System parameters}
$M$: number of UEs; $B$: number of BSs; $K$: number of beams per BS; $R$: number of discrete rate levels; $T$: time horizon; $\mathcal A$: set of all base arms, with $|\mathcal A|=MBKR$; $\mathcal S$: set of all feasible super arms (assignments); $\mathcal R=\{r_1<\cdots<r_R\}$: set of available rates; $r_{\max}:=\max_{r\in\mathcal R} r$: maximum rate.

\paragraph{Per-arm quantities}
For each base arm $i\in\mathcal A$: $X_{i,t}\in\{0,1\}$: ACK/NACK outcome at time $t$, with $\mathbb E[X_{i,t}]=\psi_i$; $\psi_i$: transmission success probability; $\mu_i := r_i \psi_i \in [0,r_{\max}]$: expected throughput; $n_i(t)$: number of times arm $i$ has been played up to time $t$; $\hat\psi_i(t)$: empirical success probability; $\hat{\psi}_{i,n}$: empirical success probability based on $n$ i.i.d.\ samples.

\paragraph{Assignment-level quantities}
A super arm $s\in\mathcal S$ includes exactly one base arm per UE, hence exactly $M$ base arms, written as $\{i_m(s)\}_{m=1}^M$.
\begin{itemize}
\item $g(s):=\frac{1}{M}\sum_{m=1}^M \mu_{i_m(s)} = \frac{1}{M}\sum_{m=1}^M r_{i_m(s)}\,\psi_{i_m(s)}$: average expected throughput of assignment $s$. When the dependence on the success-probability vector needs to be explicit, we write $g(s;\psi_s)$.
\item $\hat g_{\mathrm{shared}}(s,t):=\frac{1}{M}\sum_{m=1}^M r_{i_m(s)}\,\hat\psi_{i_m(s)}(t)$: empirical average throughput using shared counters.
\item $\Delta_{\mathrm{sat}}(s) := [\tau_r - g(s)]_+$: per-round satisficing regret of assignment $s$.
\item $s^\star \in \arg\max_{s\in\mathcal S} g(s)$: optimal assignment; $g^\star := g(s^\star)$: optimal average throughput.
\item $\Delta_s := g(s^\star) - g(s)$: standard regret gap of assignment $s$.
\item $\Delta_{\max} := \max_{s \in \mathcal S} \Delta_s$: maximum standard regret gap.
\end{itemize}

\paragraph{Concentration and index values}
\begin{itemize}
\item $c(t,n):=\sqrt{\frac{3 \log t}{2n}}$ \rev{for $n\ge 1$}: concentration radius,
\rev{with the convention $c(t,0):=+\infty$}.
\item $\mathrm{LCB}_i(t) := r_i\max\{0,\,\hat\psi_i(t)-c(t,n_i(t))\}$: lower
confidence bound index.
\item $\mathrm{MEAN}_i(t) := r_i\,\hat\psi_i(t)$: empirical mean index.
\end{itemize}
\rev{Together with the convention $\hat\psi_i(t):=0$ for $n_i(t)=0$, an unobserved
base arm has $\mathrm{LCB}_i(t)=\mathrm{MEAN}_i(t)=0$: both indices
under-estimate $\mu_i=r_i\psi_i\ge 0$, so an unobserved arm can neither make the
LCB index optimistic nor serve as an over-estimating arm for the MEAN gate.}

\paragraph{Gap quantities}
\begin{itemize}
\item $\Delta_\star := g^\star-\tau_r$: realizability margin.
\item $\Delta_{\min}^{\mathrm{std}} := \min_{s\in\mathcal S,\; s\neq s^\star}\bigl(g(s^\star)-g(s)\bigr)$: minimum standard regret gap.
\item $\mathcal S_{\mathrm{sat}} := \{s\in\mathcal S:\ g(s)\ge \tau_r\}$: satisficing assignment set.
\item $\mathcal S_{\mathrm{bad}} := \{s\in\mathcal S:\ g(s)< \tau_r\}$: bad assignment set.
\item $\mathcal S_{\mathrm{bad}}^{(i)}:=\{s\in\mathcal S_{\mathrm{bad}}: i\in s\}$: bad assignments containing arm $i$.
\item $\Delta_i^{\mathrm{bad}} := \min_{s\in\mathcal S_{\mathrm{bad}}^{(i)}} \Delta_{\mathrm{sat}}(s)$: minimum satisficing gap for arm $i$ (when $\mathcal S_{\mathrm{bad}}^{(i)}\neq\emptyset$).
\item $\Delta_i^{\mathrm{norm}} := \frac{\Delta_i^{\mathrm{bad}}}{r_i}$: normalized gap for arm $i$.
\item When $\mathcal S_{\mathrm{bad}}^{(i)}=\emptyset$, we adopt the convention $\frac{r_i^2}{2(\Delta_i^{\mathrm{bad}})^2}:=0$.
\item $\Delta_\star^{\mathrm{nr}} := \tau_r - g^\star$: non-realizability margin (used when $\tau_r > g^\star$).
\item $\Delta^{\mathrm{nr}}(s) := \tau_r - g(s)$: per-assignment non-realizability gap.
\end{itemize}

\paragraph{Events}
\begin{itemize}
\item The good event at time $t$, under which \rev{no observed base arm
over-estimates its success probability by more than the confidence radius}:
\begin{align}
\mathcal G_t:=\bigcap_{i\in\mathcal A}
\left\{\rev{\hat\psi_i(t)-\psi_i\le c(t,n_i(t))}\right\}.
\label{eqn:goodevent}
\end{align}
\item The overestimation event for base arm $i$ at time $t$ with threshold
$\varepsilon>0$:
\begin{align}
E_{i,t}(\varepsilon)
:=
\{ i\in S_t\}
\cap
\Bigl\{\hat\psi_{i,n_i(t)}\ge \psi_i+\varepsilon\Bigr\}.
\label{eqn:overestimate_def}
\end{align}
\end{itemize}

\paragraph{Auxiliary counting quantities}
\begin{itemize}
\item $S_t \in \mathcal S$: super arm played at time $t$.
\item $B_{\mathrm{LCB}}(T) := \sum_{t=\rev{1}}^{T}\mathbf 1\{S_t\in\mathcal S_{\mathrm{bad}},\;\mathcal G_t,\;\text{LCB at }t\}$: bad pulls in the LCB phase.
\item $B_{\mathrm{MEAN}}(T) := \sum_{t=\rev{1}}^{T}\mathbf 1\{S_t\in\mathcal S_{\mathrm{bad}},\;\text{MEAN at }t\}$: bad pulls in the MEAN phase.
\item $B_{\mathrm{sub}}(T) := \sum_{t=1}^T\mathbf{1}\{S_t\neq s^\star\}$: number of suboptimal plays in a fresh CTS instance.
\item $\mathcal T_i(T) := \{t\in\{\rev{1},\dots,T\}: \text{arm $i$ is played at time $t$}\}$: pull times of arm $i$.
\end{itemize}

\paragraph{CTS-phase constants}
$B_{\mathrm{cts}} := r_{\max}/M$: bounded-smoothness constant
for the CTS phase (Lemma~\ref{lem:lipschitz_constant});
$k^\star := M$; $K_{\max} := M$;
$c_{\mathrm{post}}$: universal posterior constant defined in
Lemma~\ref{lem:subopt_direct} using Lemma~5 in Appendix~A of
\cite{perrault2020statistical};
$\hat C_1$, $\hat C_0$: suboptimality constants
(specified up to a universal constant in
Lemma~\ref{lem:subopt_direct}).

\subsubsection{Assumptions}

\rev{We first assume that the satisficing threshold is strictly positive.}

\begin{assumption}[\rev{Positive satisficing threshold}]
\label{ass:positive_threshold}
\rev{The satisficing threshold satisfies
$
\tau_r > 0.
$}
\end{assumption}

The next two assumptions are related to whether $\tau_r$ can be achieved by some assignment. Our algorithm achieves different regret guarantees under these assumptions.

\begin{assumption}[Realizability with margin]
\label{ass:realizability_margin}
We assume that $\Delta_\star = g^\star-\tau_r > 0$, i.e., the optimal assignment achieves the target with positive margin.
\end{assumption}

\begin{assumption}[Non-realizability]
\label{ass:nonrealizable}
The satisficing threshold is strictly larger than the maximum achievable average throughput, i.e.,
$
\tau_r > g^\star.
$
In other words, $\Delta_\star^{\mathrm{nr}} = \tau_r-g^\star > 0$.
\end{assumption}

The next assumption is a standard simplifying assumption in CMAB theoretical analysis; see, e.g., \cite{DBLP:journals/corr/CombesLPT15}.

\begin{assumption}[Unique optimal assignment]
\label{ass:unique_optimal}
The optimal assignment is unique, i.e.,
$
\arg\max_{s\in\mathcal S} g(s)={s^\star}.
$
\end{assumption}

\begin{assumption}[\rev{Independent stationary base-arm outcomes}]
\label{ass:independent_outcomes}
\rev{For each base arm $i\in\mathcal A$, the outcomes are
i.i.d.\ Bernoulli with mean $\psi_i$ over time.
Outcomes are also mutually independent across base arms.}
\end{assumption}

Under Assumptions~\ref{ass:realizability_margin},~\ref{ass:unique_optimal}, and \ref{ass:independent_outcomes}, we have
$s^\star\in \mathcal S_{\mathrm{sat}}$ and $\Delta_{\min}^{\mathrm{std}} > 0$.

\subsubsection{Satisficing regret bound under realizable target}

When the target is realizable, SAT-CTS incurs bounded satisficing regret independent of horizon $T$.

\begin{theorem}
\label{thm:preview_realizable}
Under Assumptions~\ref{ass:realizability_margin},~\ref{ass:unique_optimal}, and \ref{ass:independent_outcomes},
the expected cumulative satisficing regret satisfies
\[
\mathbb{E}\!\left[R^S(T)\right]
\le \rev{R_{\mathrm{conf}}}
  + R_{\mathrm{MEAN}}
  + R_{\mathrm{CTS}},
\]
where $R_{\mathrm{conf}}$ is the expected satisficing regret incurred when the good event in \eqref{eqn:goodevent} fails for some $t$, and $R_{\mathrm{MEAN}}$, $R_{\mathrm{CTS}}$ represent the expected satisficing regrets incurred during MEAN and CTS phases of SAT-CTS respectively under the good event. Each term above is a finite constant independent
of~$T$, whose values are explicitly given as
\begin{align*}
R_{\mathrm{conf}}
&= \frac{\pi^2}{6} |\mathcal{A}| \tau_r,\qquad
R_{\mathrm{MEAN}}
= \tau_r \sum_{i\in\mathcal{A}}
   \frac{r_i^2}{2\,(\Delta_i^{\mathrm{bad}})^2},
\end{align*} and $R_{\mathrm{CTS}}$ is the finite constant defined in
\eqref{eq:cts_regret_constant} of Corollary~\ref{cor:cts_regret}.
\end{theorem}

The proof of Theorem \ref{thm:preview_realizable} is given in Section~\ref{sec:realizable}. This theorem shows for the first time that bounded satisficing regret can be achieved for a combinatorial semi-bandit problem. Interestingly, the expected regret during rounds in which confidence bounds fail or rounds in the MEAN phase of SAT-CTS increase linearly with the number of base arms. This scaling matches with the scaling of the gap-dependent regret with base arms in standard CMAB with linear rewards \cite{pmlr-v125-merlis20a}. The inverse dependence on $(\Delta_i^{\mathrm{bad}})^2$ in $R_{\mathrm{MEAN}}$ represents hardness of distinguishing satisficing assignments from others. Under the good event, no regret is incurred during the LCB phase. The regret coming from the CTS phase is a constant independent of $T$. \rev{Our proof uses geometrically increasing phase lengths. A committed round
of length $2^i$ both supplies enough optimal plays to make a subsequent MEAN
gate failure unlikely and doubles in length, so the probability that a further
round starts is summable over rounds.} \rev{In each CTS phase, we use the suboptimal assignment selection
bound in Lemma~\ref{lem:subopt_direct}, based on
\cite{pmlr-v80-wang18a} with the counting argument of
\cite{perrault2020statistical}.} However, it is shown in \cite[Theorem~3]{pmlr-v80-wang18a} that the regret of CTS can have exponential dependence on $M$. To the best of our knowledge, there is no known tighter bound for the regret of CTS for the assignment problem that we consider in our work. Nevertheless, we perform comprehensive experiments to show how the regret scales with $M$ and other parameters that determine the number of base arms in Section \ref{sec:experiments}.

\subsubsection{Standard regret bound under non-realizable target}

When the target is non-realizable, SAT-CTS incurs regret polylogarithmic in $T$.

\begin{theorem}
\label{thm:preview_nonrealizable}
Under Assumptions~\ref{ass:nonrealizable}, ~\ref{ass:unique_optimal}, and ~\ref{ass:independent_outcomes}, SAT-CTS achieves the following standard regret bound
\[
\mathbb{E}[R^{\mathrm{std}}(T)] = O((\log T)^2).
\]
\end{theorem}
\remove{Due to limited space, the proof of Theorem \ref{thm:preview_nonrealizable} is given in the appendix of the technical report \cite{ozyldrm2025tech}.}\add{The proof of Theorem \ref{thm:preview_nonrealizable} is given in Appendix \ref{sec:nonrealizable}.} This theorem shows that under non-realizable target, the standard regret of SAT-CTS grows almost as slow as that of CTS \cite{pmlr-v80-wang18a}, which is $O(\log T)$. Indeed, in the proof of Theorem \ref{thm:preview_nonrealizable}, we show that SAT-CTS incurs only a finite expected transient outside committed CTS rounds. After this transient, its regret is governed by the sum of the regret contributions of restarted CTS rounds, resulting in standard regret that grows polylogarithmically with the horizon \(T\). These results prove the versatility of SAT-CTS by showing that it can work favorably under both settings without knowing which setting it is facing with.

\section{Realizable Satisficing Regret Analysis} \label{sec:realizable}

In this section we upper bound the satisficing regret when $\tau_r$ is realizable. \remove{The proofs omitted due to limited space are given in the technical report \cite{ozyldrm2025tech}.}\add{Omitted proofs can be found in Appendix \ref{appendix:omitted}.}

\subsection{Technical Lemmas}

The following is a result of subgaussianity of $\hat\psi_{i,n} - \psi_i$.

\begin{lemma}
\label{ass:subgaussian}
For each arm $i\in\mathcal A$ and $n\geq 1$, the empirical success probability $\hat\psi_{i,n}$ satisfies, for all $\varepsilon>0$,
\begin{align*}
\mathbb P\!\left(\hat\psi_{i,n}\ge\psi_i+\varepsilon\right)
&\le e^{-2n\varepsilon^2}, ~~
\mathbb P\!\left(\hat\psi_{i,n}\le\psi_i-\varepsilon\right)\le e^{-2n\varepsilon^2}.
\end{align*}
\end{lemma}

Using Lemma \ref{ass:subgaussian} and union bound gives the following concentration lemma for $\hat\psi_i(t)$. 
\begin{lemma}
\label{lem:single_arm}
For any arm $i\in\mathcal A$ and any $t \geq \rev{2}$,
$\mathbb P\!\left(\rev{\hat\psi_i(t)-\psi_i>c(t,n_i(t))}\right)\le  \rev{t^{-2}}$.
\end{lemma}

The next lemma characterizes the smoothness of the expected average throughput in base arm rewards.
\begin{lemma}
\label{lem:lipschitz_constant}
For any feasible assignment $s\in\mathcal S$ and any two success-probability vectors $\psi_s,\tilde\psi_s \in [0,1]^M$,
\[
|g(s;\psi_s)-g(s;\tilde\psi_s)|
\le \frac{r_{\max}}{M}\|\psi_s-\tilde\psi_s\|_1.
\]
Hence the bounded-smoothness assumption in \cite{pmlr-v80-wang18a}
holds with $B_{\mathrm{cts}} := r_{\max}/M$.
\end{lemma}

The next lemma bounds the regret under $\mathcal G^c_t$.
\begin{lemma}
\label{lem:good_event_cost}
For \rev{every $t\ge 2$} we have
$\mathbb P(\mathcal G_t^c)\le \rev{|\mathcal A|}\,t^{-2}$, and
\[
\mathbb E\!\left[\sum_{t=\rev{1}}^{T}\Delta_{\mathrm{sat}}(S_t)\,
\mathbf 1\{\mathcal G_t^c\}\right]
\le \frac{\pi^2}{\rev{6}}\,|\mathcal A|\,\tau_r .
\]
\end{lemma}
\begin{proof}
\rev{Fix $t\ge 2$. A union bound over the $|\mathcal A|$ base arms together with
Lemma~\ref{lem:single_arm} gives $\mathbb P(\mathcal G_t^c)\le |\mathcal A|\,t^{-2}$.
Since $g(S_t)\ge 0$, we have $\Delta_{\mathrm{sat}}(S_t)\le\tau_r$ pointwise, hence}
\[
\mathbb E\!\left[\Delta_{\mathrm{sat}}(S_t)\,\mathbf 1\{\mathcal G_t^c\}\right]
\le \tau_r\,\mathbb P(\mathcal G_t^c)\le \rev{|\mathcal A|}\,\tau_r\,t^{-2}.
\]
\rev{Bounding the $t=1$ term by $\tau_r$ and summing the above display over
$t\ge 2$ using $\sum_{t\ge 2}t^{-2}=\frac{\pi^2}{6}-1$ gives}
\[
\rev{\tau_r+|\mathcal A|\,\tau_r\Bigl(\frac{\pi^2}{6}-1\Bigr)
\le \frac{\pi^2}{6}\,|\mathcal A|\,\tau_r,}
\]
\rev{where the last step uses $|\mathcal A|\ge 1$.}
\end{proof}

The next lemma proves conservatism of LCB on $\mathcal G_t$.
\begin{lemma}
\label{lem:lcb_conservative}
For any \rev{$t\ge 1$}, on $\mathcal G_t$ we have $\mu_i\ge \mathrm{LCB}_i(t)$ for
every base arm $i\in\mathcal A$.
Consequently, for any assignment $s\in\mathcal S$,
$g(s)\ \ge\ \frac{1}{M}\sum_{m=1}^{M}\mathrm{LCB}_{i_m(s)}(t)$.
\end{lemma}

\begin{proof}
\rev{If $n_i(t)=0$, then $\hat\psi_i(t)=0$ and $c(t,0)=+\infty$ by convention, so
$\mathrm{LCB}_i(t)=r_i\max\{0,-\infty\}=0\le r_i\psi_i=\mu_i$; the claim holds.}

\rev{If $n_i(t)\ge 1$, then on} $\mathcal G_t$ \rev{we have
$\hat\psi_i(t)-\psi_i\le c(t,n_i(t))$, i.e.,}
$\psi_i\ge \rev{\hat\psi_i(t)-c(t,n_i(t))}$. \rev{Combining this with $\psi_i\ge 0$ gives
$\psi_i\ge \max\{0,\hat\psi_i(t)-c(t,n_i(t))\}$}, hence
\[
\mu_i = r_i\,\psi_i
      \ge r_i\max\{0,\,\hat\psi_i(t)-c(t,n_i(t))\}
      = \mathrm{LCB}_i(t).
\]
Averaging over the $M$ components of any $s$ gives the stated inequality.
\end{proof}

\subsection{Regret Decomposition}

We work with the expected cumulative satisficing regret
$\mathbb E[R^S(T)]$. SAT-CTS operates in three
phases: LCB, MEAN, and CTS. We bound each phase's
contribution separately, using the bad pull counts $B_{\mathrm{LCB}}(T)$ and $B_{\mathrm{MEAN}}(T)$ from Section~\ref{sec:main_guarantees}. The following lemma decomposes the regret of SAT-CTS over phases. 
\begin{lemma}
\label{lem:regret_good_bad}
\begin{align*}
\mathbb E\bigl[R^S(T)\bigr]
&\le \frac{\pi^2}{\rev{6}}\,|\mathcal A|\,\tau_r
+ \tau_r\Bigl(
  \mathbb E\bigl[B_{\mathrm{LCB}}(T)\bigr]
+ \mathbb E\bigl[B_{\mathrm{MEAN}}(T)\bigr]
\Bigr)\\
&\quad+ \mathbb{E}\!\bigg[
  \sum_{t\in\mathcal{T}_{\mathrm{CTS}}}
  \!\!\Delta_{\mathrm{sat}}(S_t)\bigg],
\end{align*}
where $\mathcal{T}_{\mathrm{CTS}}$ is the set of
time steps in the CTS phase.
\end{lemma}

\begin{proof}
Write $\mathbf 1 = \mathbf 1\{\mathcal G_t\}
+ \mathbf 1\{\mathcal G_t^c\}$.
The contribution from $\mathcal G_t^c$ is at most
$\frac{\pi^2}{\rev{6}}|\mathcal A|\tau_r$ by
Lemma~\ref{lem:good_event_cost}.
On $\mathcal G_t$, bad pulls in the LCB and MEAN
phases are counted by $B_{\mathrm{LCB}}$ and
$B_{\mathrm{MEAN}}$, each costing at most $\tau_r$.
The CTS-phase regret is bounded separately since
each CTS round uses fresh priors and the
satisficing regret per round is controlled by the
standard CTS regret (Section~\ref{sec:cts_phase}).
\com{
Notation fitted version:
\begin{align*}
& \mathbb E[R^S(T)] \\
&= \sum_{t=1}^{T} \mathbb E\!\left[\Delta_{\mathrm{sat}}(S_t)\right] \\
&=
   \sum_{t=1}^{T}
     \mathbb E\!\left[\Delta_{\mathrm{sat}}(S_t)
     \mathbf 1\{\mathcal G_t\}\right] \\
& + \sum_{t=1}^{T}
     \mathbb E\!\left[\Delta_{\mathrm{sat}}(S_t)
     \mathbf 1\{\mathcal G_t^c\}\right] \\
& \le \frac{\pi^2}{6}\,|\mathcal A|\,\tau_r
   + \tau_r\,\mathbb E[B_{\mathrm{LCB}}(T)] \\
& + \tau_r\,\mathbb E[B_{\mathrm{MEAN}}(T)]
   + \mathbb E\!\bigg[
     \sum_{t\in\mathcal T_{\mathrm{CTS}}}
     \Delta_{\mathrm{sat}}(S_t)
     \bigg].
\end{align*}
}
\end{proof}

Next, we show that there are no bad pulls in the LCB phase.
\begin{lemma}
\label{lem:no_bad_lcb}
\rev{For every $T\ge 1$} , $B_{\mathrm{LCB}}(T)=0$ a.s.
\end{lemma}
\begin{proof}
On $\mathcal G_t$, if $s$ is selected in
the LCB phase then
$
g(s)\ge\frac{1}{M}\sum_m\mathrm{LCB}_{i_m(s)}(t)
\ge\tau_r
$
by Lemma~\ref{lem:lcb_conservative}, so $s\notin\mathcal S_{\mathrm{bad}}$.
\end{proof}

By Lemma~\ref{lem:no_bad_lcb}, the regret in Lemma~\ref{lem:regret_good_bad}
simplifies to
\begin{align}\label{eqn:regret_simplified}
\begin{split}
\mathbb E\bigl[R^S(T)\bigr]
&\le \frac{\pi^2}{\rev{6}}\,|\mathcal A|\,\tau_r
+ \tau_r\,\mathbb E\bigl[B_{\mathrm{MEAN}}(T)\bigr]\\
&\quad+ \mathbb{E}\!\bigg[
  \sum_{t\in\mathcal{T}_{\mathrm{CTS}}}
  \!\!\Delta_{\mathrm{sat}}(S_t)\bigg].
\end{split}
\end{align}
It remains to bound
$\mathbb E[B_{\mathrm{MEAN}}(T)]$
(Section~\ref{sec:mean_phase}) and
the CTS-phase regret
(Section~\ref{sec:cts_phase}).

\subsection{Bounding $\mathbb E[B_{\mathrm{MEAN}}(T)]$}
\label{sec:mean_phase}

Recall the overestimation event $E_{i,t}(\varepsilon)$ from \eqref{eqn:overestimate_def} and the gap quantities $\Delta_i^{\mathrm{bad}}$, $\Delta_i^{\mathrm{norm}}$, $\mathcal S_{\mathrm{bad}}^{(i)}$ from Section~\ref{sec:main_guarantees}. The following lemma relates bad arm pulls in the MEAN phase with the overestimation events. 
\begin{lemma}
\label{lem:mean_bad_implies_single}
On any round $t$ in the MEAN phase,
\[
\{S_t\in\mathcal S_{\mathrm{bad}},\ \hat g_{\mathrm{shared}}(S_t,t)\ge \tau_r\}
\ \subseteq\
\bigcup_{m=1}^{M} E_{i_m(S_t),t}\!\left(\Delta_{i_m(S_t)}^{\mathrm{norm}}\right).
\]
\end{lemma}

Next, we bound $B_{\mathrm{MEAN}}(T)$ in terms of the events in \eqref{eqn:overestimate_def}.

\begin{lemma}
\label{lem:Bmean_dominate}
$
B_{\mathrm{MEAN}}(T)
\le
\sum_{i\in\mathcal A}\sum_{t=\rev{1}}^{T}\mathbf 1\{E_{i,t}(\Delta_i^{\mathrm{norm}})\}
$.
\end{lemma}

\begin{proof}
By Lemma~\ref{lem:mean_bad_implies_single}, every MEAN-phase bad pull implies
$E_{i,t}(\Delta_i^{\mathrm{norm}})$ for at least one base arm $i$ used at time $t$. Hence
$
\mathbf 1\{S_t\in\mathcal S_{\mathrm{bad}},\ \text{MEAN phase}\}
\le
\sum_{i\in\mathcal A}\mathbf 1\{E_{i,t}(\Delta_i^{\mathrm{norm}})\}
$.
Summing over $t=\rev{1},\dots,T$ yields the claim.
\end{proof}

Recall the pull-time set $\mathcal T_i(T)$ from Section~\ref{sec:main_guarantees}. The next lemma bounds occurrences of $E_{i,t}(\varepsilon)$ in terms of deviations of $\hat\psi_{i,n}$.
\begin{lemma}
\label{lem:decouple_counts}
Fix $i\in\mathcal A$ and $\varepsilon>0$. Then
$
\sum_{t=\rev{1}}^{T}\mathbf 1\{E_{i,t}(\varepsilon)\}
=
\sum_{t\in\mathcal T_i(T)}\mathbf 1\{E_{i,t}(\varepsilon)\}
\le
\sum_{n=1}^{T}\mathbf 1\Bigl\{\hat\psi_{i,n}\ge \psi_i+\varepsilon\Bigr\}
$.
\end{lemma}

Using Lemma \ref{lem:decouple_counts}, the next lemma bounds the expected occurrences of $E_{i,t}(\varepsilon)$.
\begin{lemma}
\label{lem:one_arm_geo}
Fix $i\in\mathcal A$ and $\varepsilon>0$. 
$
\mathbb E\!\left[\sum_{t=\rev{1}}^{T}\mathbf 1\{E_{i,t}(\varepsilon)\}\right]
\le
\frac{1}{2\varepsilon^2}
$.
\end{lemma}

\begin{proof}
By Lemma~\ref{lem:decouple_counts} and linearity of expectation,
\[
\mathbb E\!\left[\sum_{t=\rev{1}}^{T}\mathbf 1\{E_{i,t}(\varepsilon)\}\right]
\le
\sum_{n=1}^{T}\mathbb P\!\left(\hat\psi_{i,n}\ge \psi_i+\varepsilon\right).
\]
By Lemma~\ref{ass:subgaussian},
$
\mathbb P\!\left(\hat\psi_{i,n}\ge \psi_i+\varepsilon\right)\le e^{-2n\varepsilon^2}
$.
Hence
\[
\sum_{n=1}^{T}\mathbb P\!\left(\hat\psi_{i,n}\ge \psi_i+\varepsilon\right)
\le
\sum_{n=1}^{\infty} e^{-2n\varepsilon^2}
=
\frac{1}{e^{2\varepsilon^2}-1}
\le
\frac{1}{2\varepsilon^2},
\]
which follows from $e^x-1\ge x$ with $x=2\varepsilon^2$.
\end{proof}

Equipped with the results above, we bound $\mathbb E[B_{\mathrm{MEAN}}(T)]$ in the following lemma. 
\begin{lemma}
\label{lem:bound_Bmean_clean}
$
\mathbb E\bigl[B_{\mathrm{MEAN}}(T)\bigr]
\le
\sum_{i\in\mathcal A}
\frac{r_i^2}{2\,(\Delta_i^{\mathrm{bad}})^2}
$.
\end{lemma}

\begin{proof}
By Lemma~\ref{lem:Bmean_dominate} and linearity of expectation,
\[
\mathbb E[B_{\mathrm{MEAN}}(T)]
\le
\sum_{i\in\mathcal A}\mathbb E\!\left[\sum_{t=\rev{1}}^{T}\mathbf 1\{E_{i,t}(\Delta_i^{\mathrm{norm}})\}\right].
\]
Apply Lemma~\ref{lem:one_arm_geo} to each $i$ with
$
\Delta_i^{\mathrm{norm}}=\frac{\Delta_i^{\mathrm{bad}}}{r_i}
$,
to obtain
\[
\mathbb E[B_{\mathrm{MEAN}}(T)]
\le
\sum_{i\in\mathcal A} \frac{1}{2(\Delta_i^{\mathrm{norm}})^2}
=
\sum_{i\in\mathcal A} \frac{r_i^2}{2\,(\Delta_i^{\mathrm{bad}})^2}.
\]
Arms with $\mathcal S_{\mathrm{bad}}^{(i)}=\emptyset$ contribute zero by convention ($\frac{r_i^2}{2(\Delta_i^{\mathrm{bad}})^2}:=0$).
\end{proof}

\subsection{CTS-Phase Analysis}
\label{sec:cts_phase}

The CTS phase runs in rounds $i=1,2,3,\ldots$ with a
doubling schedule: round~$i$ runs a fresh CTS instance
(initialised with $\mathrm{Beta}(1,1)$ priors) for
exactly $2^i$ steps.
All observations update the shared counters used by the
MEAN gate.
After round~$i$, if any assignment satisfies
$\hat g_{\mathrm{shared}}(s,t)\ge\tau_r$, the MEAN gate
takes over; otherwise round~$i+1$ begins.
We analyse this phase following the geometric-series
framework of SELECT~\cite{feng2025satisficing},
treating each fresh CTS round as the oracle.

Recall that under Assumption~\ref{ass:unique_optimal},
$s^\star$ is the unique optimal super arm. For a fresh CTS
instance on a deterministic horizon~$T$, recall $B_{\mathrm{sub}}(T)$ from Section~\ref{sec:main_guarantees}. Let
\[
\mathrm{Reg}^{\mathrm{std}}_{\mathrm{CTS}}(T)
:=
\mathbb E\!\left[\sum_{t=1}^{T}\bigl(g(s^\star)-g(S_t)\bigr)\right]
\]
denote the standard regret of a fresh CTS instance over a deterministic horizon $T$.
By Lemma~\ref{lem:lipschitz_constant}, the bounded-smoothness
assumption in \cite{pmlr-v80-wang18a} holds with
$B_{\mathrm{cts}}=r_{\max}/M$.

\rev{The next lemma bounds the expected number of suboptimal assignment
plays in a fresh CTS instance using the four-event decomposition of
\cite{pmlr-v80-wang18a} and the corrected counting argument in
\cite{perrault2020statistical}.}

\begin{lemma}
\label{lem:subopt_direct}
For a fresh CTS instance on deterministic horizon $T$ and under assumption \ref{ass:independent_outcomes}, letting
$B_{\mathrm{cts}} := r_{\max}/M$ (Lemma~\ref{lem:lipschitz_constant}), for any
$
0<\varepsilon<
\frac{M\,\Delta_{\min}^{\mathrm{std}}}
{2\,r_{\max}\,\bigl((k^\star)^2+2\bigr)}
$,
we have
$
\mathbb{E}[B_{\mathrm{sub}}(T)]
\;\le\;
\hat{C}_1 \log T + \hat{C}_0
$,
where
\begin{align*}
\hat{C}_1
&:=
\sum_{i\in\mathcal{A}}
\max_{s:\,i\in s,\,s\neq s^\star}
\frac{8B_{\mathrm{cts}}^2\,|s|^2}
     {\bigl(\Delta_s
      -2B_{\mathrm{cts}}((k^\star)^2+2)\varepsilon
      \bigr)^2}, \\
\hat{C}_0
&:=
\frac{|\mathcal{A}|K_{\max}^2}{\varepsilon^2}
+4|\mathcal{A}|\\
&\quad+
\frac{17}{\varepsilon^2}
\left[
\left(1+\frac{c_{\mathrm{post}}}{\varepsilon^4}\right)^{k^\star}
-1
\right],
\end{align*}
with $B_{\mathrm{cts}} = r_{\max}/M$, $k^\star=M$, $K_{\max}=M$,
and $c_{\mathrm{post}}:=\max\{c,c'\}$, where $c,c'>0$ are the
universal constants in Lemma~5 of Appendix~A of
\cite{perrault2020statistical}. 
\end{lemma}

\rev{
Recall that $\Delta_\star:=g^\star-\tau_r>0$.
We derive a concentration bound for any integer
sample threshold $n\ge1$.
}

The next lemma relates the MEAN gate failure with deviation of the base arm rewards.
\begin{lemma}
\label{lem:gate_to_arm}
For any time $t$ at which every constituent arm
of $s^\star$ has been pulled at least once,
\[
\{\hat g_{\mathrm{shared}}(s^\star,t)<\tau_r \}
\subseteq
\bigcup_{m=1}^{M}
\left\{
\hat\psi_{i_m(s^\star)}(t)
-\psi_{i_m(s^\star)}
<-\frac{\Delta_\star}{r_{\max}}
\right\}
\]
\end{lemma}

\rev{Using Lemma~\ref{lem:gate_to_arm}, the next lemma
bounds the probability of a MEAN gate failure at any
time when every constituent arm of $s^\star$ has
at least $n$ shared observations.}

\begin{lemma}
\label{lem:sstar_conc}
\rev{Fix any integer $N\ge1$.}
If every constituent arm $i_m(s^\star)$, $m\in[M]$,
has been observed at least \rev{$N$} times via the shared
counters at some time $t_0$, then
\[
\mathbb{P}\!\left(\exists\,t\ge t_0:\;
\hat g_{\mathrm{shared}}(s^\star,t)<\tau_r\right)
\rev{\le
\frac{M\exp\!\left(-2N\Delta_\star^2/r_{\max}^2\right)}
{1-\exp\!\left(-2\Delta_\star^2/r_{\max}^2\right)}},
\]
where $\hat g_{\mathrm{shared}}(s^\star,t)$ is evaluated
using the shared counters at time $t$.
\rev{The time $t_0$ may be random; the probability
above is unconditional.}
\end{lemma}

\begin{proof}
Since $\hat g_{\mathrm{shared}}(s^\star,t)$ changes only
when a constituent arm of $s^\star$ is pulled, it suffices
to take the union over pull counts rather than
time steps.

At any time $t\ge t_0$ where the gate fails,
each constituent arm $i_m(s^\star)$ has some
pull count $n_{i_m(s^\star)}(t)\ge \rev{N}$.
By Lemma~\ref{lem:gate_to_arm}, gate failure at
time $t$ implies that at least one arm $i_m(s^\star)$
satisfies
$
\hat\psi_{i_m(s^\star)}(t)-\psi_{i_m(s^\star)}
<-\frac{\Delta_\star}{r_{\max}}
$.
Therefore:
\begin{align*}
&\{\exists\,t\ge t_0:\;
\hat g_{\mathrm{shared}}(s^\star,t)<\tau_r\}\\
&\subseteq
\bigcup_{m=1}^{M}
\bigcup_{n=\rev{N}}^{\infty}
\!\left\{
\hat\psi_{i_m(s^\star),n}-\psi_{i_m(s^\star)}
<-\frac{\Delta_\star}{r_{\max}}
\right\}.
\end{align*}
\rev{Here $\hat\psi_{i_m(s^\star),n}$ is the empirical
mean of the first $n$ observations in that arm's
i.i.d. sample sequence.}
Taking probabilities and applying the union
bound:
\begin{align*}
&\mathbb{P}\!\left(\exists\,t\ge t_0:\;
\hat g_{\mathrm{shared}}(s^\star,t)<\tau_r\right)\\
&\le \sum_{m=1}^{M}\sum_{n=\rev{N}}^{\infty}
\mathbb{P}\!\left(
\hat\psi_{i_m(s^\star),n}-\psi_{i_m(s^\star)}
<-\frac{\Delta_\star}{r_{\max}}\right)\\
&\le \sum_{m=1}^{M}\sum_{n=\rev{N}}^{\infty}
\exp\!\left(-\frac{2n\Delta_\star^2}{r_{\max}^2}
\right)
\tag{Lemma~\ref{ass:subgaussian}}\\
&= M\cdot
\frac{\exp\!\left(
-\frac{2\rev{N}\Delta_\star^2}{r_{\max}^2}\right)}
{1-\exp\!\left(
-\frac{2\Delta_\star^2}{r_{\max}^2}\right)}.
\end{align*}
\end{proof}

\begin{lemma}
\label{lem:sufficient_optimal_plays}
\rev{For a fresh CTS instance on any deterministic
integer horizon $T\ge2$, and any integer $1\le N<T$,}
\[
\rev{
\mathbb{P}(N_{s^\star}(T)<N)
\le \frac{\hat C_1\log T+\hat C_0}{T-N},
}
\]
where $N_{s^\star}(T)$ denotes the number of times
the unique optimal arm $s^\star$ is played
\rev{within that instance up to time~$T$}.
\end{lemma}

\begin{proof}
Since $s^\star$ is the unique optimal arm,
every play is either optimal or suboptimal, so
$N_{s^\star}(T)=T-B_{\mathrm{sub}}(T)$.
Hence
\begin{align*}
\mathbb{P}(N_{s^\star}(T)<\rev{N})
&=\mathbb{P}(B_{\mathrm{sub}}(T)>T-\rev{N})
\le \frac{\mathbb{E}[B_{\mathrm{sub}}(T)]}{T-\rev{N}}\\
&\le \frac{\hat C_1\log T+\hat C_0}{T-\rev{N}},
\end{align*}
where the second step is Markov's inequality
\rev{and the third uses Lemma~\ref{lem:subopt_direct}.}
\end{proof}

The key result given below proves that CTS phases become rarer over time. 
\begin{proposition}
\label{prop:round_termination}
Under Assumptions~\ref{ass:realizability_margin},~\ref{ass:unique_optimal}, and \ref{ass:independent_outcomes},
\rev{for every CTS round index $i\ge1$,}
\begin{align*}
&\rev{\mathbb{P}(\text{round }i+1\text{ starts})}\\
&\rev{\le
\mathbb{P}(\text{round }i\text{ starts})
\frac{2(\hat C_1 i\log 2+\hat C_0)}{2^i}}\\
&\rev{\quad+
\frac{M\exp\!\left(-2^i\Delta_\star^2/r_{\max}^2\right)}
{1-\exp\!\left(-2\Delta_\star^2/r_{\max}^2\right)}.}
\end{align*}
\end{proposition}

\begin{proof}
\rev{Consider the infinite-horizon execution and fix
$i\ge1$. Write $A_i:=\{\text{round }i\text{ starts}\}$.
If $\mathbb{P}(A_i)=0$, then
$\mathbb{P}(A_{i+1})=0$ and the claim is immediate.
Otherwise, set $T=2^i$ and $N=2^{i-1}$, and let
$N_{s^\star}(T)$ count optimal plays within round~$i$
when that round starts. Define}
\[
\rev{
B:=\{\text{round }i\text{ starts and }
N_{s^\star}(T)<N\}.
}
\]
\rev{Then $B\subseteq A_i$. Decomposing according to
$B$ and taking unconditional probabilities gives}
\begin{align*}
\rev{\mathbb{P}(A_{i+1})}
&\rev{=
\mathbb{P}(A_{i+1}\cap B)
+\mathbb{P}(A_{i+1}\cap B^c)}\\
&\rev{\le
\mathbb{P}(B)+\mathbb{P}(A_{i+1}\cap B^c)}\\
&\rev{=
\mathbb{P}(A_i)\mathbb{P}(B\mid A_i)
+\mathbb{P}(A_{i+1}\cap B^c).}
\end{align*}

\rev{For the first term, conditional on $A_i$,
the reset priors, fresh algorithmic randomness,
and future i.i.d. feedback give a fresh CTS
instance of length $T$.
Lemma~\ref{lem:sufficient_optimal_plays} therefore gives}
\[
\rev{
\mathbb{P}(B\mid A_i)
\le
\frac{\hat C_1\log T+\hat C_0}{T-N}
=
\frac{2(\hat C_1 i\log 2+\hat C_0)}{2^i}.
}
\]

\rev{For the second term, on $A_{i+1}\cap B^c$,
round~$i$ has completed with at least $N$ optimal
plays. Thus, every constituent arm of $s^\star$
has at least $N$ shared observations at every
subsequent gate evaluation, since the shared
counts never decrease.}

Round~$i+1$ starts only if the MEAN gate fails
after round~$i$, i.e., if at the gate-evaluation time $t$
$\hat g_{\mathrm{shared}}(s,t)<\tau_r$
for all $s\in\mathcal S$.
Since $s^\star\in\mathcal S$, this implies in particular
that
$
\hat g_{\mathrm{shared}}(s^\star,t)<\tau_r
$.
\rev{Consequently, on $A_{i+1}\cap B^c$, the time
$t_0$ defined in Lemma~\ref{lem:sstar_conc} has been
reached by the end of round~$i$, and the subsequent
MEAN gate failure occurs at some $t\ge t_0$.
Lemma~\ref{lem:sstar_conc} therefore yields}
\begin{align*}
\rev{\mathbb{P}(A_{i+1}\cap B^c)}
&\rev{\le
\frac{M\exp\!\left(-2N\Delta_\star^2/r_{\max}^2\right)}
{1-\exp\!\left(-2\Delta_\star^2/r_{\max}^2\right)}}\\
&\rev{=
\frac{M\exp\!\left(-2^i\Delta_\star^2/r_{\max}^2\right)}
{1-\exp\!\left(-2\Delta_\star^2/r_{\max}^2\right)}.}
\end{align*}
\rev{Combining the two bounds proves the claim.}
\end{proof}

Next, we bound the CTS-phase regret of SAT-CTS.
\begin{corollary}
\label{cor:cts_regret}
Under Assumptions~\ref{ass:realizability_margin},~\ref{ass:independent_outcomes} and \ref{ass:unique_optimal}, the expected cumulative
satisficing regret incurred during the CTS phase satisfies
\[
\mathbb{E}\!\bigg[
  \sum_{t\in\mathcal{T}_{\mathrm{CTS}}}
  \Delta_{\mathrm{sat}}(S_t)
\bigg]
\rev{\le R_{\mathrm{CTS}},}
\]
\rev{where}
\begin{align}
\rev{R_{\mathrm{CTS}}:=\tau_r\Bigg[}
&\rev{\hat C_1\log 2+\hat C_0}
\nonumber\\
&\rev{+\sum_{i=1}^{\infty}
\bigl(\hat C_1(i+1)\log 2+\hat C_0\bigr)}
\nonumber\\
&\rev{\quad\times\left(
\frac{2(\hat C_1 i\log 2+\hat C_0)}{2^i}
\right.}
\nonumber\\
&\rev{\qquad\left.
+\frac{M\exp\!\left(-2^i\Delta_\star^2/r_{\max}^2\right)}
{1-\exp\!\left(-2\Delta_\star^2/r_{\max}^2\right)}
\right)\Bigg].}
\label{eq:cts_regret_constant}
\end{align}
\rev{For each fixed problem instance, $R_{\mathrm{CTS}}$
is finite and independent of $T$. The constants
$\hat C_1$ and $\hat C_0$ depend on the standard optimality
gaps. For small realizability margins, the factor
$(1-e^{-2\Delta_\star^2/r_{\max}^2})^{-1}$
is approximately $r_{\max}^2/(2\Delta_\star^2)$.}
\end{corollary}

\subsection{Proof of Theorem \ref{thm:preview_realizable}}
\begin{align*}
\mathbb{E}[R^S(T)]
&\le\frac{\pi^2}{\rev{6}}\,|\mathcal{A}|\,\tau_r
   + \tau_r\,\mathbb{E}[B_{\mathrm{MEAN}}(T)]\\
&\quad+ \mathbb{E}\!\bigg[
   \sum_{t\in\mathcal{T}_{\mathrm{CTS}}}
   \!\!\Delta_{\mathrm{sat}}(S_t)\bigg] 
\end{align*}
by \eqref{eqn:regret_simplified}. The \rev{first} term is bounded by $R_{\mathrm{conf}}$ by
Lemma~\ref{lem:good_event_cost}.
The \rev{second} term is bounded by $R_{\mathrm{MEAN}}$ by
Lemma~\ref{lem:bound_Bmean_clean}.
The \rev{third} term is bounded by $R_{\mathrm{CTS}}$ by
Corollary~\ref{cor:cts_regret}.
All \rev{three} terms are finite constants independent of~$T$.

\rev{
\subsection{Discussion on the Theoretical Analysis}
\label{sec:theory_discussion}
Our theoretical guarantees are gap dependent and assume stationary
success probabilities under independent block fading. Assignments close
to the satisficing boundary can enlarge the corresponding finite-time
constants, while near ties between the optimal and suboptimal assignments
enlarge the constants that depend on the standard optimality gaps.
Ties for the optimum are excluded by Assumption~\ref{ass:unique_optimal}.
Nonstationary or temporally correlated channels require additional
assumptions, modeling, and analysis, which we leave to future work;
contextual extensions are relevant when observable side information is
available.
}

\section{Experiments}
\label{sec:experiments}
\subsection{Experimental Setup}
\label{subsec:experimental_setup}
We consider a \rev{scalable} multi-cell mmWave system with \rev{$B$ BSs and $M$ single-antenna UEs, where the network size is varied to conduct scalability experiments.}\footnote{Our implementation is available at https://github.com/Bilkent-CYBORG/Satisficing-with-Binary-Feedback-for-Combinatorial-Beam-Alignment.} The large-scale channel parameters are obtained from the DeepMIMO
\texttt{city\_3\_houston\_28} scenario at $28$ GHz
\cite{alkhateeb2019deepmimogenericdeeplearning}. \rev{Unless stated otherwise, each BS is equipped with $N=64$ antennas and
$N_{\mathrm{RF}}=8$ RF chains, and the receiver noise figure is $F_m=7$~dB for all UEs. Each active beam is allocated a fixed transmit power of $p_b=1$~W, independent of the number of UEs served by the BS. Each served UE is allocated four PRBs (each with 12 subcarriers) at 120~kHz subcarrier spacing, giving a per-UE sub-channel bandwidth of $B_{\mathrm{sc}}=5.76$~MHz, and each block contains 624 modulation symbols carrying coded data. The main configuration has $M=15$ UEs,
$B=3$ BSs and $K=120$ beams per BS.}
\reve{We use the first $B$ BS entries of the loaded dataset. The nominal UE locations were randomly selected once and stored in the implementation. Each nominal location is mapped deterministically to the nearest unused DeepMIMO grid point having at least one ray-traced path from every selected BS. The resulting UE locations remain fixed throughout the entire simulation unless noted otherwise.}\footnote{\reve{The nominal coordinates and the seed for selecting additional UEs when $M>15$ are specified in \texttt{src/obs/config.py} of our implementation.}}\rev{ Hence, the large-scale geometry and average path powers remain fixed across runs, while randomness arises from the independently generated small-scale fading coefficients and the resulting ACK/NACK outcomes.}

\reve{All BSs use the standard beamsteering codebook
\cite{wei2022fast,Codebook2017}, obtained by evaluating the steering vector
in \eqref{eqn:steering} with $d=\lambda/2$ at $K$ direction cosines spaced
uniformly over $[-1,1)$:}
\begin{align}
\reve{
\mathbf f_{b,k}=\mathbf a(u_k),
\qquad
u_k=-1+\frac{2(k-1)}{K},
\qquad k\in[K].
}
\end{align}

\rev{For this particular configuration, $g^\star=5.5737$~b/s/Hz. For each BS UE link, we
use $L=3$ path sparse mmWave channel which is also used in the beam alignment literature~\cite{fastbeam}. Their
angles and average powers are obtained from DeepMIMO. For the randomness on the channel, the complex path coefficients in~\eqref{eqn:channelgain} are generated as}
\begin{align}
\beta_{m,b,\ell}(t)
=
\sqrt{P_{m,b,\ell}}\,
\varepsilon_{m,b,\ell}(t),
\label{eqn:measured_path_gain}
\end{align}
\rev{where $P_{m,b,\ell}$ denotes the average power of the $\ell$th path from BS $b$ to UE $m$, obtained directly from the DeepMIMO. The small-scale fading is modeled by $\varepsilon_{m,b,\ell}(t) \overset{\mathrm{i.i.d.}}{\sim} \mathcal{CN}(0,1)$, which are independently redrawn in each decision round.}

\rev{The path geometry and average powers remain fixed over the horizon which comes from DeepMIMO, while the
small-scale coefficients are independently redrawn in each decision round,
following the Gaussian complex path coefficient model of~\cite{Ayach_2014}.
The rate set is obtained from physical downlink shared channel (PDSCH) MCS Table~2 of
TS~38.214~\cite{3gpp38214}. We use MCS indices
$\{9,15,21,24\}$, corresponding to}
\begin{align}
\mathcal R
=
\{2.41,\;3.90,\;5.55,\;6.57\}
\quad\text{bits/s/Hz}.
\end{align}

\rev{The satisficing threshold $\tau_r$ is an input to the learner rather than a
quantity it estimates, and in practice it would be set by the service
requirement, expressed in the same units as the rate set. Since the operator does not know $g^\star$ in advance, a threshold
set optimistically may turn out to be non realizable, which is why we analyze
both cases and why SAT-CTS needs no knowledge of which one it faces. The BLER curves associated with the selected MCSs are generated offline using
NVIDIA Sionna~\cite{sionna}. The BLER curves are generated over a static complex AWGN channel at each fixed SNR. The tabulated curves and their generation settings are
available with our implementation.}

\subsection{Results and Discussion}

\rev{Throughout this section, SAT-CTS (reset) denotes Algorithm~\ref{alg:lcb-cts} with the
committed CTS rounds initialized from $\mathrm{Beta}(1,1)$ priors, as in the
theoretical analysis, while SAT-CTS (retain) uses the globally accumulated
posterior during these rounds. Both variants are reported to separate the
analyzed algorithm from the experimental implementation.}

\rev{All tables report the same quantities. Results are averaged over independent
runs with common random numbers across methods, with standard deviations
reported where indicated: the cumulative satisficing regret $R^S(T)$, the average throughput $\bar g(t)$ in b/s/Hz, i.e., the running mean of the
network-average throughput over rounds $1,\dots,t$ (written $\bar g$ when
evaluated at $t=T$), Jain's fairness index
$\mathcal J$ of the per-UE throughputs, and the 
fraction of UEs meeting the
target, ``$\ge\tau_r$''. Letting $G_m(t):=\sum_{s\le t}\rho_{m,s}x_{m,s}$ denote
the throughput delivered to UE~$m$ up to round~$t$, Jain's fairness
index~\cite{jain1984fairness} is}
\begin{align}
\mathcal J
:=
\frac{\Bigl(\sum_{m=1}^{M} G_m(T)\Bigr)^{2}}
     {M\sum_{m=1}^{M} G_m(T)^{2}}
\in\Bigl[\tfrac1M,1\Bigr],
\label{eqn:jain}
\end{align}
\rev{which equals one when all UEs receive the same throughput and decreases as the
throughput becomes concentrated on fewer UEs. For the metric ``$\ge\tau_r$'', at
each round we compute the fraction of UEs whose average throughput up to that
round meets $\tau_r$, and report the average of these fractions over the whole
horizon. In experiments where $M$, $B$, $K$ and $NF$ change,
$g^\star$ changes with the configuration, so the target is set to
$\tau_r=0.85\,g^\star$ in each row. This holds the target-to-optimum ratio
fixed across rows; the absolute margin and the satisficing
gaps still vary with the configuration.}

\rev{Table~\ref{tab:tau-sensitivity} shows the effect of changing the satisficing threshold under the main configuration.
For realizable targets with a sufficiently large margin, both SAT-CTS variants
achieve substantially lower satisficing regret than CTS and CUCB. As
$\tau_r$ approaches $g^\star$, the margin decreases and the problem becomes
more difficult, resulting in higher regret, as predicted by
Theorem~\ref{thm:preview_realizable}. The behavior of regret over time is illustrated in
Fig.~\ref{fig:regret-realizable}, which shows the first $10^4$ rounds with the
vertical axis clipped at $2\times10^4$. 

When the target is non-realizable, no
assignment can satisfy the threshold and all methods therefore incur regret
that grows with the horizon. In this regime, SAT-CTS (retain) and CTS exhibit nearly identical performance.
This is expected: under a non-realizable target both gates fire in very limited number, so both
SAT-CTS variants spend almost the whole horizon inside committed CTS rounds.
Since the retaining variant carries its posterior across CTS rounds, its
empirical behavior closely tracks that of continuously run CTS; our theoretical
guarantees apply to the reset variant. The reset
variant executes the same policy but discards its posterior at every committed
round, so it pays the relearning cost once per round what produces the
$O((\log T)^2)$ bound of Theorem~\ref{thm:preview_nonrealizable}. This behaviour can be seen in Fig.~\ref{fig:regret-standard-nr} which shows the standard regret directly at the non-realizable
target.}

\begin{table}[!t]
\centering
\caption{\rev{Threshold sensitivity: satisficing regret $R^S(T)$, main configuration ($T=3\times10^4$, 100 runs). NR: non-realizable target.}}
\label{tab:tau-sensitivity}
\scriptsize
\setlength{\tabcolsep}{2pt}
\begin{tabular}{@{}c cccc@{}}
\toprule
$\tau_r$ & \makecell{SAT-CTS\\(reset)} & \makecell{SAT-CTS\\(retain)} & CTS & CUCB\\
\midrule
2.79 & $57\!\pm\!26$ & $51\!\pm\!22$ & $630\!\pm\!11$ & $14{,}254\!\pm\!126$\\
3.90 & $217\!\pm\!92$ & $205\!\pm\!87$ & $1{,}155\!\pm\!18$ & $32{,}908\!\pm\!139$\\
4.74 & $1{,}189\!\pm\!611$ & $849\!\pm\!273$ & $1{,}918\!\pm\!24$ & $55{,}067\!\pm\!88$\\
4.88 & $2{,}036\!\pm\!1{,}100$ & $1{,}357\!\pm\!461$ & $2{,}544\!\pm\!32$ & $59{,}241\!\pm\!92$\\
5.02 & $4{,}310\!\pm\!1{,}785$ & $1{,}960\!\pm\!373$ & $3{,}456\!\pm\!58$ & $63{,}435\!\pm\!93$\\
5.16 & $8{,}964\!\pm\!1{,}500$ & $2{,}780\!\pm\!390$ & $4{,}783\!\pm\!74$ & $67{,}634\!\pm\!93$\\
5.85$^{\mathrm{NR}}$ & $41{,}444\!\pm\!135$ & $20{,}145\!\pm\!71$ & $20{,}155\!\pm\!73$ & $88{,}334\!\pm\!93$\\
\bottomrule
\end{tabular}
\end{table}

\begin{figure}[!ht]
  \centering
  \includegraphics[width=\columnwidth]{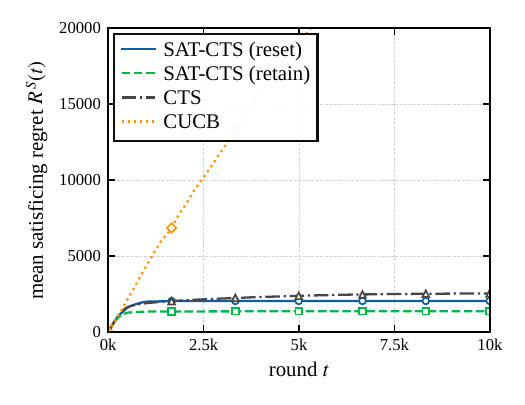}
  \caption{\rev{Cumulative satisficing regret at a realizable target
  $\tau_r=4.88$, $T=3\times10^4$, 100 runs.}}
  \label{fig:regret-realizable}
\end{figure}

\rev{The standard regret results in Table~\ref{tab:tau-sensitivity-std} further
illustrate the distinction between
satisficing and maximization. SAT-CTS can stop searching once an
assignment satisfying $\tau_r$ has been identified and therefore need not
learn the throughput-maximizing assignment. This reduces satisficing regret while causing larger standard regret when the target is below
$g^\star$ the standard regret of CTS grows
logarithmically and flattens, whereas that of the SAT-CTS variants continues to
accumulate at a constant rate once a satisficing assignment has been found.}

\begin{table}[!t]
\centering
\caption{\rev{Threshold sensitivity: standard regret $R^{\mathrm{std}}(T)$, for the
same runs as Table~\ref{tab:tau-sensitivity}.}}
\label{tab:tau-sensitivity-std}
\scriptsize
\setlength{\tabcolsep}{2pt}
\begin{tabular}{@{}c cccc@{}}
\toprule
$\tau_r$ & \makecell{SAT-CTS\\(reset)} & \makecell{SAT-CTS\\(retain)} & CTS & CUCB\\
\midrule
2.79 & $64{,}624\!\pm\!10{,}664$ & $67{,}179\!\pm\!10{,}332$ & $11{,}866\!\pm\!79$ & $80{,}045\!\pm\!93$\\
3.90 & $38{,}156\!\pm\!7{,}152$ & $37{,}938\!\pm\!6{,}786$ & $11{,}855\!\pm\!78$ & $80{,}045\!\pm\!93$\\
4.74 & $21{,}978\!\pm\!2{,}766$ & $21{,}776\!\pm\!2{,}730$ & $11{,}853\!\pm\!76$ & $80{,}045\!\pm\!93$\\
4.88 & $18{,}991\!\pm\!1{,}628$ & $17{,}592\!\pm\!1{,}984$ & $11{,}850\!\pm\!73$ & $80{,}045\!\pm\!93$\\
5.02 & $18{,}126\!\pm\!1{,}518$ & $15{,}200\!\pm\!1{,}964$ & $11{,}857\!\pm\!78$ & $80{,}045\!\pm\!93$\\
5.16 & $18{,}661\!\pm\!1{,}293$ & $11{,}790\!\pm\!1{,}904$ & $11{,}863\!\pm\!77$ & $80{,}045\!\pm\!93$\\
5.85$^{\mathrm{NR}}$ & $33{,}155\!\pm\!135$ & $11{,}855\!\pm\!71$ & $11{,}865\!\pm\!73$ & $80{,}045\!\pm\!93$\\
\bottomrule
\end{tabular}
\end{table}

\begin{figure}[!ht]
  \centering
  \includegraphics[width=\columnwidth]{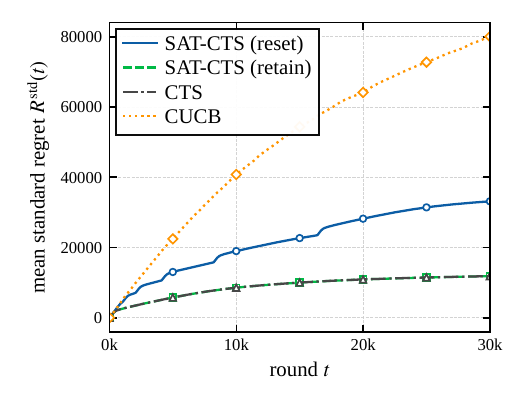}
  \caption{\rev{Cumulative \emph{standard} regret at the non-realizable target
  $\tau_r=5.85$, $T=3\times10^4$, 100 runs.}}
  \label{fig:regret-standard-nr}
\end{figure}

\rev{Fig.~\ref{fig:throughput} reports the same experiment as an average throughput, which is the quantity the threshold is stated in. At this target, both SAT-CTS
variants settle slightly above $\tau_r$, whereas CTS and CUCB do not read
$\tau_r$ and therefore behave identically at every target. Reading along the
horizon, the average is already flat at $10^4$ rounds. This is consistent with
the horizon-free bound of Theorem~\ref{thm:preview_realizable}, whose direct
evidence is the flattening of the cumulative satisficing regret in
Fig.~\ref{fig:regret-realizable}.}

\begin{figure}[!t]
  \centering
  \includegraphics[width=\columnwidth]{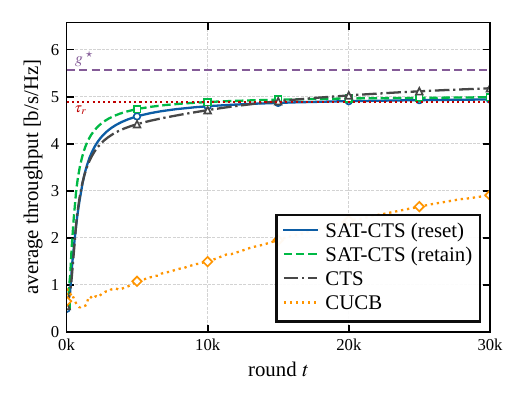}
  \caption{\rev{Average throughput against time at the realizable target
  $\tau_r=4.88$, $T=3\times10^4$, 100 runs.}}
  \label{fig:throughput}
\end{figure}

Tables~\ref{tab:scale-m}--\ref{tab:scale-k} evaluate scalability with respect
to the numbers of UEs, BSs, and beams. \reve{In these experiments, the per-BS RF-chain count is scaled with the network
size as $N_{\mathrm{RF},b}=\min\{N,\lceil 1.5M/B\rceil\}$, so that the total
RF-chain budget is at least the number of UEs in every tested configuration
while the per-BS count never exceeds the number of antennas $N=64$. In the $M=192$ configuration, $N_{\mathrm{RF},b}=N=64$, which is outside
the hybrid front-end regime ($N_{\mathrm{RF},b}\ll N$) described in
Section~III-A. These experiments assess learning scalability and do not
establish hardware scalability in that regime.} Across the considered configurations,
SAT-CTS (reset) and SAT-CTS (retain) consistently achieve much lower satisficing regret
than CUCB and generally improve over CTS, at an average throughput
$\bar g$ within $85$--$90\%$ of the corresponding $g^\star$.

\begin{table}[!ht]
\centering
\caption{\rev{Scaling the number of UEs $M$ $B=3$, $K=120$, $T=10^4$, 100 runs.}}
\label{tab:scale-m}
\scriptsize
\setlength{\tabcolsep}{2pt}
\begin{tabular}{@{}c c l r@{\,} c c c@{}}
\toprule
$M$ & $g^\star$ & Method & $R^S(T)$ & $\bar g$ & $\mathcal J$ & $\ge\tau_r$\\
\midrule
\multirow{4}{*}{15} & \multirow{4}{*}{5.574} & SAT-CTS (reset) & $1{,}020\!\pm\!471$ & $4.757$ & $0.762\!\pm\!0.023$ & $0.670\!\pm\!0.027$\\
 &  & SAT-CTS (retain) & $868\!\pm\!312$ & $4.790$ & $0.763\!\pm\!0.024$ & $0.681\!\pm\!0.019$\\
 &  & CTS & $1{,}907\!\pm\!23$ & $4.717$ & $0.771\!\pm\!0.003$ & $0.627\!\pm\!0.002$\\
 &  & CUCB & $32{,}464\!\pm\!38$ & $1.491$ & $0.659\!\pm\!0.002$ & $0.000\!\pm\!0.000$\\
\cmidrule(l){3-7}
\multirow{4}{*}{20} & \multirow{4}{*}{5.821} & SAT-CTS (reset) & $548\!\pm\!278$ & $5.095$ & $0.810\!\pm\!0.019$ & $0.734\!\pm\!0.032$\\
 &  & SAT-CTS (retain) & $541\!\pm\!210$ & $5.097$ & $0.809\!\pm\!0.019$ & $0.734\!\pm\!0.026$\\
 &  & CTS & $1{,}715\!\pm\!21$ & $5.115$ & $0.828\!\pm\!0.002$ & $0.671\!\pm\!0.002$\\
 &  & CUCB & $32{,}656\!\pm\!35$ & $1.683$ & $0.735\!\pm\!0.001$ & $0.000\!\pm\!0.000$\\
\cmidrule(l){3-7}
\multirow{4}{*}{50} & \multirow{4}{*}{5.949} & SAT-CTS (reset) & $424\!\pm\!141$ & $5.198$ & $0.828\!\pm\!0.015$ & $0.746\!\pm\!0.019$\\
 &  & SAT-CTS (retain) & $411\!\pm\!130$ & $5.224$ & $0.830\!\pm\!0.015$ & $0.753\!\pm\!0.027$\\
 &  & CTS & $1{,}624\!\pm\!13$ & $5.340$ & $0.860\!\pm\!0.001$ & $0.697\!\pm\!0.001$\\
 &  & CUCB & $32{,}744\!\pm\!20$ & $1.772$ & $0.766\!\pm\!0.001$ & $0.000\!\pm\!0.000$\\
\cmidrule(l){3-7}
\multirow{4}{*}{100} & \multirow{4}{*}{5.969} & SAT-CTS (reset) & $314\!\pm\!97$ & $5.211$ & $0.835\!\pm\!0.012$ & $0.745\!\pm\!0.021$\\
 &  & SAT-CTS (retain) & $294\!\pm\!70$ & $5.198$ & $0.832\!\pm\!0.014$ & $0.744\!\pm\!0.020$\\
 &  & CTS & $1{,}575\!\pm\!8$ & $5.385$ & $0.869\!\pm\!0.001$ & $0.700\!\pm\!0.001$\\
 &  & CUCB & $31{,}369\!\pm\!15$ & $1.922$ & $0.778\!\pm\!0.001$ & $0.000\!\pm\!0.000$\\
\cmidrule(l){3-7}
\multirow{4}{*}{192} & \multirow{4}{*}{6.047} & SAT-CTS (reset) & $278\!\pm\!28$ & $5.204$ & $0.841\!\pm\!0.010$ & $0.736\!\pm\!0.017$\\
 &  & SAT-CTS (retain) & $279\!\pm\!42$ & $5.240$ & $0.845\!\pm\!0.011$ & $0.745\!\pm\!0.017$\\
 &  & CTS & $1{,}552\!\pm\!6$ & $5.506$ & $0.891\!\pm\!0.001$ & $0.709\!\pm\!0.001$\\
 &  & CUCB & $31{,}524\!\pm\!11$ & $1.941$ & $0.777\!\pm\!0.000$ & $0.000\!\pm\!0.000$\\
\bottomrule
\end{tabular}
\end{table}

\begin{table}[!ht]
\centering
\caption{\rev{Scaling the number of BSs $B$ $M=15$, $K=120$, $T=10^4$, 100 runs.}}
\label{tab:scale-b}
\scriptsize
\setlength{\tabcolsep}{2pt}
\begin{tabular}{@{}c c l r@{\,} c c c@{}}
\toprule
$B$ & $g^\star$ & Method & $R^S(T)$ & $\bar g$ & $\mathcal J$ & $\ge\tau_r$\\
\midrule
\multirow{4}{*}{1} & \multirow{4}{*}{2.360} & SAT-CTS (reset) & $58\!\pm\!39$ & $2.147$ & $0.338\!\pm\!0.001$ & $0.332\!\pm\!0.001$\\
 &  & SAT-CTS (retain) & $47\!\pm\!27$ & $2.142$ & $0.338\!\pm\!0.002$ & $0.333\!\pm\!0.000$\\
 &  & CTS & $191\!\pm\!6$ & $2.174$ & $0.339\!\pm\!0.001$ & $0.331\!\pm\!0.000$\\
 &  & CUCB & $5{,}312\!\pm\!21$ & $1.479$ & $0.331\!\pm\!0.000$ & $0.301\!\pm\!0.001$\\
\cmidrule(l){3-7}
\multirow{4}{*}{2} & \multirow{4}{*}{4.173} & SAT-CTS (reset) & $906\!\pm\!475$ & $3.446$ & $0.547\!\pm\!0.011$ & $0.507\!\pm\!0.014$\\
 &  & SAT-CTS (retain) & $609\!\pm\!146$ & $3.473$ & $0.545\!\pm\!0.014$ & $0.516\!\pm\!0.005$\\
 &  & CTS & $921\!\pm\!17$ & $3.530$ & $0.571\!\pm\!0.002$ & $0.508\!\pm\!0.001$\\
 &  & CUCB & $19{,}427\!\pm\!26$ & $1.551$ & $0.520\!\pm\!0.000$ & $0.038\!\pm\!0.001$\\
\cmidrule(l){3-7}
\multirow{4}{*}{3} & \multirow{4}{*}{5.574} & SAT-CTS (reset) & $1{,}020\!\pm\!471$ & $4.757$ & $0.762\!\pm\!0.023$ & $0.670\!\pm\!0.027$\\
 &  & SAT-CTS (retain) & $868\!\pm\!312$ & $4.790$ & $0.763\!\pm\!0.024$ & $0.681\!\pm\!0.019$\\
 &  & CTS & $1{,}907\!\pm\!23$ & $4.717$ & $0.771\!\pm\!0.003$ & $0.627\!\pm\!0.002$\\
 &  & CUCB & $32{,}464\!\pm\!38$ & $1.491$ & $0.659\!\pm\!0.002$ & $0.000\!\pm\!0.000$\\
\bottomrule
\end{tabular}
\end{table}

\rev{The scalability results in Tables~\ref{tab:scale-m} and~\ref{tab:scale-k} show that the relative benefit of SAT-CTS becomes more pronounced as the number of UEs or the codebook size grows. Increasing the codebook size from $K=30$ to $K=240$ raises the satisficing regret of CTS from $589$ to $3739$, whereas the regret of SAT-CTS (retain) increases from $329$ to $1125$. This difference is consistent with the satisficing objective. A similar trend is observed when the number of UEs is increased. At $M=100$, for example, CTS incurs a satisficing regret of $1575$, compared with $294$ for SAT-CTS (retain), a reduction of approximately $81\%$. In the BS-scaling experiment (Table~\ref{tab:scale-b}), SAT-CTS achieves lower satisficing regret than CTS and CUCB at every $B$, although the relative reduction does not increase monotonically with $B$.}

\begin{table}[!ht]
\centering
\caption{\rev{Scaling the codebook size $K$ $M=15$, $B=3$, $T=10^4$, 100 runs.}}
\label{tab:scale-k}
\scriptsize
\setlength{\tabcolsep}{2pt}
\begin{tabular}{@{}c c l r@{\,} c c c@{}}
\toprule
$K$ & $g^\star$ & Method & $R^S(T)$ & $\bar g$ & $\mathcal J$ & $\ge\tau_r$\\
\midrule
\multirow{4}{*}{30} & \multirow{4}{*}{5.252} & SAT-CTS (reset) & $506\!\pm\!190$ & $4.531$ & $0.731\!\pm\!0.008$ & $0.679\!\pm\!0.031$\\
 &  & SAT-CTS (retain) & $329\!\pm\!91$ & $4.569$ & $0.732\!\pm\!0.011$ & $0.694\!\pm\!0.030$\\
 &  & CTS & $589\!\pm\!26$ & $4.920$ & $0.831\!\pm\!0.001$ & $0.681\!\pm\!0.005$\\
 &  & CUCB & $11{,}838\!\pm\!49$ & $3.239$ & $0.734\!\pm\!0.001$ & $0.085\!\pm\!0.001$\\
\cmidrule(l){3-7}
\multirow{4}{*}{60} & \multirow{4}{*}{5.521} & SAT-CTS (reset) & $1{,}135\!\pm\!416$ & $4.743$ & $0.765\!\pm\!0.022$ & $0.660\!\pm\!0.027$\\
 &  & SAT-CTS (retain) & $643\!\pm\!167$ & $4.782$ & $0.759\!\pm\!0.024$ & $0.689\!\pm\!0.017$\\
 &  & CTS & $985\!\pm\!25$ & $5.001$ & $0.824\!\pm\!0.001$ & $0.677\!\pm\!0.002$\\
 &  & CUCB & $23{,}069\!\pm\!42$ & $2.382$ & $0.733\!\pm\!0.001$ & $0.000\!\pm\!0.000$\\
\cmidrule(l){3-7}
\multirow{4}{*}{120} & \multirow{4}{*}{5.574} & SAT-CTS (reset) & $1{,}020\!\pm\!471$ & $4.757$ & $0.762\!\pm\!0.023$ & $0.670\!\pm\!0.027$\\
 &  & SAT-CTS (retain) & $868\!\pm\!312$ & $4.790$ & $0.763\!\pm\!0.024$ & $0.681\!\pm\!0.019$\\
 &  & CTS & $1{,}907\!\pm\!23$ & $4.717$ & $0.771\!\pm\!0.003$ & $0.627\!\pm\!0.002$\\
 &  & CUCB & $32{,}464\!\pm\!38$ & $1.491$ & $0.659\!\pm\!0.002$ & $0.000\!\pm\!0.000$\\
\cmidrule(l){3-7}
\multirow{4}{*}{240} & \multirow{4}{*}{5.585} & SAT-CTS (reset) & $1{,}192\!\pm\!674$ & $4.762$ & $0.766\!\pm\!0.022$ & $0.660\!\pm\!0.042$\\
 &  & SAT-CTS (retain) & $1{,}125\!\pm\!535$ & $4.766$ & $0.765\!\pm\!0.023$ & $0.665\!\pm\!0.031$\\
 &  & CTS & $3{,}739\!\pm\!30$ & $4.428$ & $0.741\!\pm\!0.002$ & $0.522\!\pm\!0.003$\\
 &  & CUCB & $36{,}963\!\pm\!18$ & $1.051$ & $0.602\!\pm\!0.002$ & $0.000\!\pm\!0.000$\\
\bottomrule
\end{tabular}
\end{table}

\begin{table}[!ht]
\centering
\caption{\rev{Sensitivity to the receiver noise figure NF, in dB, in the main
configuration $T=10^4$, 100 runs.}}
\label{tab:nf-sweep}
\scriptsize
\setlength{\tabcolsep}{2pt}
\begin{tabular}{@{}c c l r@{\,} c c c@{}}
\toprule
$NF$ & $g^\star$ & Method & $R^S(T)$ & $\bar g$ & $\mathcal J$ & $\ge\tau_r$\\
\midrule
\multirow{4}{*}{5} & \multirow{4}{*}{5.694} & SAT-CTS (reset) & $967\!\pm\!618$ & $4.928$ & $0.800\!\pm\!0.026$ & $0.667\!\pm\!0.037$\\
 &  & SAT-CTS (retain) & $843\!\pm\!370$ & $4.964$ & $0.802\!\pm\!0.026$ & $0.676\!\pm\!0.023$\\
 &  & CTS & $2{,}056\!\pm\!28$ & $4.849$ & $0.801\!\pm\!0.002$ & $0.626\!\pm\!0.002$\\
 &  & CUCB & $31{,}684\!\pm\!32$ & $1.672$ & $0.677\!\pm\!0.002$ & $0.000\!\pm\!0.000$\\
\cmidrule(l){3-7}
\multirow{4}{*}{7} & \multirow{4}{*}{5.574} & SAT-CTS (reset) & $1{,}020\!\pm\!471$ & $4.757$ & $0.762\!\pm\!0.023$ & $0.670\!\pm\!0.027$\\
 &  & SAT-CTS (retain) & $868\!\pm\!312$ & $4.790$ & $0.763\!\pm\!0.024$ & $0.681\!\pm\!0.019$\\
 &  & CTS & $1{,}907\!\pm\!23$ & $4.717$ & $0.771\!\pm\!0.003$ & $0.627\!\pm\!0.002$\\
 &  & CUCB & $32{,}464\!\pm\!38$ & $1.491$ & $0.659\!\pm\!0.002$ & $0.000\!\pm\!0.000$\\
\cmidrule(l){3-7}
\multirow{4}{*}{9} & \multirow{4}{*}{5.459} & SAT-CTS (reset) & $1{,}043\!\pm\!468$ & $4.662$ & $0.745\!\pm\!0.015$ & $0.668\!\pm\!0.029$\\
 &  & SAT-CTS (retain) & $855\!\pm\!271$ & $4.700$ & $0.745\!\pm\!0.016$ & $0.681\!\pm\!0.019$\\
 &  & CTS & $1{,}913\!\pm\!29$ & $4.621$ & $0.749\!\pm\!0.002$ & $0.627\!\pm\!0.002$\\
 &  & CUCB & $32{,}875\!\pm\!42$ & $1.352$ & $0.642\!\pm\!0.002$ & $0.000\!\pm\!0.000$\\
\cmidrule(l){3-7}
\multirow{4}{*}{11} & \multirow{4}{*}{5.339} & SAT-CTS (reset) & $1{,}115\!\pm\!616$ & $4.594$ & $0.736\!\pm\!0.009$ & $0.657\!\pm\!0.035$\\
 &  & SAT-CTS (retain) & $889\!\pm\!320$ & $4.609$ & $0.732\!\pm\!0.006$ & $0.669\!\pm\!0.027$\\
 &  & CTS & $1{,}924\!\pm\!30$ & $4.555$ & $0.739\!\pm\!0.001$ & $0.622\!\pm\!0.003$\\
 &  & CUCB & $33{,}186\!\pm\!31$ & $1.220$ & $0.639\!\pm\!0.002$ & $0.000\!\pm\!0.000$\\
\cmidrule(l){3-7}
\multirow{4}{*}{13} & \multirow{4}{*}{5.223} & SAT-CTS (reset) & $1{,}100\!\pm\!507$ & $4.517$ & $0.728\!\pm\!0.006$ & $0.652\!\pm\!0.035$\\
 &  & SAT-CTS (retain) & $947\!\pm\!312$ & $4.576$ & $0.730\!\pm\!0.005$ & $0.671\!\pm\!0.026$\\
 &  & CTS & $1{,}926\!\pm\!38$ & $4.491$ & $0.732\!\pm\!0.001$ & $0.611\!\pm\!0.005$\\
 &  & CUCB & $33{,}528\!\pm\!43$ & $1.086$ & $0.629\!\pm\!0.003$ & $0.000\!\pm\!0.000$\\
\bottomrule
\end{tabular}
\end{table}

\rev{Table~\ref{tab:nf-sweep} examines the sensitivity of the algorithms to the receiver noise figure. Varying $\mathrm{NF}$ shifts the effective SNR while leaving the remaining physical parameters unchanged; the target is rescaled as $\tau_r=0.85\,g^\star$ in each row, so the absolute target decreases with $\mathrm{NF}$. Across the considered range, the optimal achievable utility $g^\star$ changes only modestly, from $5.694$ to $5.223$. Increasing the noise figure affects both the ACK/NACK observation statistics and the expected throughputs of the available assignments. Nevertheless, the relative performance of the algorithms remains stable across the considered sweep.}

\reve{We also evaluated sensitivity to imperfect ACK/NACK feedback over \(T=10^5\) rounds with \(\tau_r=4.0\) b/s/Hz. Each decoding bit was independently flipped with probability \(1-p\), for \(p\in\{1.0,0.9,0.8,0.7\}\). The algorithms used the reported bits without error correction, and performance was evaluated using the true decoding model. The mean cumulative satisficing regret of SAT-CTS (reset) increased from \(186\pm102\) at \(p=1.0\) to \(42{,}535\pm20{,}150\) at \(p=0.7\) (mean \(\pm\) standard deviation). At \(p=0.7\), SAT-CTS (retain) and CTS incurred \(20{,}031\pm6{,}712\) and \(28{,}423\pm364\), respectively. The reset variant’s average throughput was still \(4.478\) b/s/Hz, above the target, because average throughput can exceed the target even when the positive per-round shortfalls in (10) accumulate substantial regret. These results show that SAT-CTS can be sensitive to feedback errors. The regret guarantees in Section V assume that the learner observes the true decoding bits.}

\FloatBarrier

\section{Conclusion and Future Research}\label{sec:conclusion}
We formulated a multi-user beam–rate selection problem with ACK/NACK feedback as a satisficing combinatorial bandit with a \rev{network averaged} per-user throughput threshold, for which we proposed SAT-CTS, an if-gated policy that blends conservative and exploratory indices. We established finite-time regret bounds for SAT-CTS: when the satisficing threshold is realizable, we proved a horizon-free upper bound on cumulative satisficing regret; when the threshold is non-realizable, we showed that SAT-CTS reduces to CTS \rev{with a geometric term}. \rev{In DeepMIMO based experiments, SAT-CTS (retain)
achieved lower mean cumulative satisficing regret than
CTS and CUCB across all tested realizable targets and
closely tracked CTS at the tested non-realizable target.
The reset variant also improved satisficing regret in
most of the configurations. Our theoretical guarantees apply to
the reset variant.}

\reve{The empirical findings are limited to the evaluated configurations within the DeepMIMO \texttt{city\_3\_houston\_28} propagation scenario under stationary block fading. Within each configuration, the large-scale geometry is fixed and small-scale fading is independently drawn across decision rounds. Evaluation across independent propagation scenarios remains future work.}

As future work, it is possible to extend our formulations and approaches to contextual combinatorial bandits that exploit side information (e.g., geometry, mobility) to accelerate learning and improve robustness. \rev{More generally, temporally correlated and non-stationary channels arising from mobility, dynamic blockage, or evolving propagation conditions motivate Markovian and non-stationary extensions of the proposed framework.} \rev{Another important direction is to establish finite-time guarantees for SAT-CTS (Retain). A further direction is to replace the network-average target with per-UE
minimum-rate or outage constraints, which would render the satisficing objective
non-separable across the selected arms and therefore call for a different oracle
and a corresponding extension of the regret analysis.} \rev{Finally, relaxing the orthogonal resource allocation considered in this work would lead to interference included, SINR-aware formulations in which rewards are no longer separable across the selected arms, requiring corresponding extensions of the learning model and regret analysis.}

\appendices

\setlength{\floatsep}{6pt}
\setlength{\textfloatsep}{6pt}
\setlength{\intextsep}{6pt}

\section{Omitted Proofs of Results in Section \ref{sec:realizable}} \label{appendix:omitted}

\subsection{Proof of Lemma \ref{lem:single_arm}}
\rev{If $n_i(t)=0$, then $c(t,n_i(t))=c(t,0)=+\infty$ by convention, so the event
$\bigl\{\big|\hat\psi_i(t)-\psi_i\big|>c(t,n_i(t))\bigr\}$ is empty and the
claim holds trivially. It therefore suffices to consider $n_i(t)\ge 1$.}

For a fixed $n \geq 1$, we have by Lemma \ref{ass:subgaussian}
\begin{align*}
    \mathbb P\left(\big|\hat\psi_{i,n}-\psi_i\big|>c(t,n)\right) &\leq 2 \exp \bigl(-2n c(t,n)^2\bigr) \\ 
    &= 2 \exp \bigl(-3 \log t \bigr) \\
    &= 2 t^{-3} .
\end{align*}
Note that on the event $\{n_i(t)\ge 1\}$ we have $1 \leq n_i(t) \leq t$. Thus, 
\begin{align*}
    \left\{ \big|\hat\psi_i(t)-\psi_i\big|>c(t,n_i(t)) \right\} \subseteq \bigcup_{n=1}^t \left\{ \big|\hat\psi_{i,n} -\psi_i\big|>c(t,n) \right\}
\end{align*}
Using the two displays above, we obtain the following via the union bound
\begin{align*}
  &\mathbb P \left(\big|\hat\psi_i(t)-\psi_i\big|>c(t,n_i(t))\right) \\
  &\qquad \qquad \leq \sum_{n=1}^t \mathbb P\left(\big|\hat\psi_{i,n}-\psi_i\big|>c(t,n)\right) \\
  &\qquad \qquad \le \sum_{n=1}^t 2 t^{-3} = 2 t^{-2}.
\end{align*}
\rev{The same argument applied to the upper deviation alone uses only the first
inequality of Lemma~\ref{ass:subgaussian}, which drops the factor $2$ and gives
\[
\mathbb P\left(\hat\psi_i(t)-\psi_i>c(t,n_i(t))\right)\le t^{-2}.
\]
Since $\mathcal G_t$ in \eqref{eqn:goodevent} constrains only the upper
deviation, this is the form used in the sequel.}

\subsection{Proof of Lemma \ref{lem:lipschitz_constant}}

We have
\begin{align*}
|g(s;\psi_s)-g(s;\tilde\psi_s)|
&=
\left|
\frac{1}{M}\sum_{m=1}^{M}
r_{i_m(s)}\bigl(\psi_{i_m(s)}-\tilde\psi_{i_m(s)}\bigr)
\right| \\
&\le
\frac{1}{M}\sum_{m=1}^{M}
r_{i_m(s)}\left|\psi_{i_m(s)}-\tilde\psi_{i_m(s)}\right| \\
&\le
\frac{r_{\max}}{M}\sum_{m=1}^{M}
\left|\psi_{i_m(s)}-\tilde\psi_{i_m(s)}\right| \\
&=
\frac{r_{\max}}{M}\|\psi_s-\tilde\psi_s\|_1.
\end{align*}

\subsection{Proof of Lemma \ref{lem:mean_bad_implies_single}}

Let $S_t=s\in\mathcal S_{\mathrm{bad}}$ and $\hat g_{\mathrm{shared}}(s,t)\ge\tau_r$.
Then
\[
\hat g_{\mathrm{shared}}(s,t)-g(s)\ge \tau_r-g(s)=\Delta_{\mathrm{sat}}(s).
\]
Write
\[
\hat g_{\mathrm{shared}}(s,t)-g(s)
=
\frac{1}{M}\sum_{m=1}^{M} r_{i_m(s)}\bigl(\hat\psi_{i_m(s)}(t)-\psi_{i_m(s)}\bigr).
\]
If for all $m$ we had
\[
\hat\psi_{i_m(s)}(t)-\psi_{i_m(s)}<\Delta^{\text{norm}}_{i_m(s)}
= \frac{\Delta_{i_m(s)}^{\mathrm{bad}}}{r_{i_m(s)}},
\]
then since $\Delta_{i_m(s)}^{\mathrm{bad}}\le\Delta_{\mathrm{sat}}(s)$ for each $i_m(s)\in s$,
\[
\hat g_{\mathrm{shared}}(s,t)-g(s)
<
\frac{1}{M}\sum_{m=1}^{M} r_{i_m(s)}\cdot\frac{\Delta_{\mathrm{sat}}(s)}{r_{i_m(s)}}
=\Delta_{\mathrm{sat}}(s),
\]
contradiction. Hence at least one $m$ satisfies
\[
\hat\psi_{i_m(s)}(t)-\psi_{i_m(s)}\ge \Delta^{\text{norm}}_{i_m(s)}.
\]
\rev{Moreover, any such $m$ necessarily has $n_{i_m(s)}(t)\ge 1$: if
$n_{i_m(s)}(t)=0$ then $\hat\psi_{i_m(s)}(t)=0$ by convention, so
\[
\hat\psi_{i_m(s)}(t)-\psi_{i_m(s)}
=-\psi_{i_m(s)}\le 0<\Delta^{\text{norm}}_{i_m(s)},
\]
and the displayed deviation cannot hold. Thus an unobserved base arm never
witnesses a MEAN-phase bad pull, and the events
$E_{i_m(s),t}(\Delta^{\text{norm}}_{i_m(s)})$ appearing in the statement
involve only arms with at least one observation.}

\subsection{Proof of Lemma \ref{lem:decouple_counts}}

\rev{First observe that $E_{i,t}(\varepsilon)$ cannot occur when arm $i$ has not yet
been observed: if $n_i(t)=0$, then $\hat\psi_{i,0}=0$ by convention, so
\[
\hat\psi_{i,n_i(t)}-\psi_i=-\psi_i\le 0<\varepsilon ,
\]
and the second event in \eqref{eqn:overestimate_def} fails. Hence
$\mathbf 1\{E_{i,t}(\varepsilon)\}=0$ whenever $n_i(t)=0$, and only rounds with
$n_i(t)\ge 1$ contribute below.}

As $t$ ranges over $\mathcal T_i(T)$, the count $n_i(t)$ takes each of its
successive values exactly once. Therefore
\[
\begin{split}
\sum_{t=1}^{T}\mathbf 1\{E_{i,t}(\varepsilon)\}
&=
\sum_{t\in\mathcal T_i(T)}\mathbf 1\{E_{i,t}(\varepsilon)\} \\
&\le
\sum_{n=0}^{n_i(T)}
\mathbf 1\Bigl\{\hat\psi_{i,n}\ge \psi_i+\varepsilon\Bigr\} \\
&=
\sum_{n=1}^{n_i(T)}
\mathbf 1\Bigl\{\hat\psi_{i,n}\ge \psi_i+\varepsilon\Bigr\} \\
&\le
\sum_{n=1}^{T}\mathbf 1\Bigl\{\hat\psi_{i,n}\ge \psi_i+\varepsilon\Bigr\},
\end{split}
\]
\rev{where the first equality holds because $E_{i,t}(\varepsilon)$ requires
$i\in S_t$, the third drops the vanishing $n=0$ term by the observation above,
and the last uses $n_i(T)\le T$.}

\subsection{Proof of Lemma \ref{lem:subopt_direct}}

Let $\theta_i(t) \sim \mathrm{Beta}(A_i, B_i)$ denote the Thompson sample for arm $i$ at time $t$, and let $\theta_{S_t}(t) := (\theta_{i_1(S_t)}(t), \ldots, \theta_{i_M(S_t)}(t))$.
We apply the four-event decomposition of the
supplementary material of \cite{pmlr-v80-wang18a}
directly to $\mathbf{1}\{S_t\neq s^\star\}$ rather
than to the regret-weighted quantity
$\Delta_{S_t}\cdot\mathbf{1}\{S_t\neq s^\star\}$.
Recall the four events:
\begin{itemize}
  \item $A(t)=\{S_t\notin\mathrm{OPT}\}$,
  \item $B(t)=\{\exists\,i\in S_t,\,
        |\hat\psi_i(t)-\psi_i|>\varepsilon/|S_t|\}$,
  \item $C(t)=\{\|\theta_{S_t}(t)-\psi_{S_t}\|_1
        >\Delta_{S_t}/B_{\mathrm{cts}}
        -((k^\star)^2+1)\varepsilon\}$,
  \item $D(t)=\{\exists\,i\in S_t,\,
        |\theta_i(t)-\hat\psi_i(t)|
        >\sqrt{2\log T/n_i(t)}\}$.
\end{itemize}

\medskip\noindent
\textbf{Term~1} ($B(t)\wedge A(t)$, 
Section~B.3.1 of \cite{pmlr-v80-wang18a}):
By Lemma~3 of \cite{pmlr-v80-wang18a}, the expected 
number of rounds arm $i$ is played with empirical 
mean deviation $>\varepsilon/K_{\max}$ is at most 
$1+K_{\max}^2/\varepsilon^2$.
Summing over arms and dropping the $\Delta_{\max}$ 
weight:
\[
\sum_{t=1}^T
\mathbb{E}[\mathbf{1}\{B(t)\wedge A(t)\}]
\;\le\;
|\mathcal{A}|
+\frac{|\mathcal{A}|K_{\max}^2}{\varepsilon^2}.
\]

\medskip\noindent
\textbf{Term~2} ($\neg B\wedge C\wedge D\wedge A$,
Section~B.3.2 of \cite{pmlr-v80-wang18a}):
By Lemma~4 of \cite{pmlr-v80-wang18a},
$\Pr[|\theta_i(t)-\hat\psi_i(t)|
>\sqrt{2\log T/n_i(t)}]\le 2/T$.
A union bound over arms gives:
\[
\sum_{t=1}^T
\mathbb{E}[\mathbf{1}\{\neg B(t) \wedge C(t) \wedge D(t) \wedge A(t) \}]
\;\le\; 2|\mathcal{A}|.
\]
\medskip\noindent
\textbf{Term~3} ($\neg B\wedge C\wedge\neg D\wedge A$,
Section~B.3.3 of \cite{pmlr-v80-wang18a}):
Define the sufficiency threshold
\[
L_i(s)
:=
\frac{2\log T}
{\Bigl(
  \frac{\Delta_s}{2B_{\mathrm{cts}}|s|}
  -\frac{(k^\star)^2+2}{|s|}\varepsilon
\Bigr)^2}
=
\frac{8B_{\mathrm{cts}}^2|s|^2\log T}
     {(\Delta_s-2B_{\mathrm{cts}}((k^\star)^2+2)\varepsilon)^2}.
\]
By Section~B.3.3 of \cite{pmlr-v80-wang18a},
if $n_i(t)>L_i(S_t)$ for all $i\in S_t$ then
$\neg B\wedge C\wedge\neg D\wedge A$ cannot occur.
Hence at any such round there exists a witness arm
$i\in S_t$ with $n_i(t)\le L_i(S_t)\le L_i^{\max}$,
where $L_i^{\max}:=\max_{s:\,i\in s,\,s\neq s^\star}
L_i(s)$.
Since $n_i(t)$ increments with each pull of arm $i$,
the values satisfying $n_i(t)\le L_i^{\max}$ can occur
for at most $\lfloor L_i^{\max}\rfloor+1
\le L_i^{\max}+1$ pulls of arm $i$.
Therefore,
\[
\sum_{t=1}^T
\mathbb{E}[\mathbf{1}\{
  \neg B(t)\wedge C(t)\wedge\neg D(t)\wedge A(t)\}]
\le
\sum_{i\in\mathcal A}(L_i^{\max}+1)
=
\hat C_1\log T+|\mathcal A|.
\]

\rev{
\medskip\noindent
\textbf{Term~4} ($\neg C\wedge A$):
By Lemma~1 of \cite{pmlr-v80-wang18a},
\[
\{\neg C(t)\wedge A(t)\}
\subseteq
\bigcup_{\varnothing\ne Z\subseteq s^\star}
\{E_{Z,1}(\theta(t))\wedge E_{Z,2}(\theta(t))\},
\]
where $E_{Z,1}$ and $E_{Z,2}$ are the events defined there.
Here $E_{Z,1}$ depends only on samples outside $Z$, while
$E_{Z,1}\wedge\neg E_{Z,2}$ implies that all arms in $Z$
are observed. Thus the corrected interval-count argument
of Lemmas~4--5 in Appendix~A of
\cite{perrault2020statistical} applies, with
$\mu_i^\star=\psi_i$ and $A^\star=s^\star$.

For each nonempty $Z\subseteq s^\star$, the expected count
between the $q$th and $(q+1)$st occurrences of
$E_{Z,1}\wedge\neg E_{Z,2}$ is bounded by
$(c\varepsilon^{-4})^{|Z|}$ for all $q\ge0$, and by
$e^{-\varepsilon^2q/8}(c'\varepsilon^{-4})^{|Z|}$
for $q>8/\varepsilon^2$.

Set $Q=\lfloor8/\varepsilon^2\rfloor+1$ and
$c_{\mathrm{post}}=\max\{c,c'\}$.
Since $0<\varepsilon<1$, we have
$Q\le9/\varepsilon^2$ and
$\sum_{q=Q}^\infty e^{-\varepsilon^2q/8}
\le8/\varepsilon^2$.
Summing over intervals and nonempty subsets gives
\begin{align*}
&\sum_{t=1}^T
\mathbb E[\mathbf1\{\neg C(t)\wedge A(t)\}]\\
&\quad\le
\frac{17}{\varepsilon^2}
\sum_{\varnothing\ne Z\subseteq s^\star}
(c_{\mathrm{post}}\varepsilon^{-4})^{|Z|}\\
&\quad=
\frac{17}{\varepsilon^2}
\left[
\left(1+\frac{c_{\mathrm{post}}}{\varepsilon^4}\right)^{k^\star}
-1
\right].
\end{align*}

\medskip\noindent
Summing all four terms yields the stated bound.

}

\subsection{Proof of Lemma \ref{lem:gate_to_arm}}

Suppose for contradiction that for all $m\in[M]$,
\[
\hat\psi_{i_m(s^\star)}(t)
-\psi_{i_m(s^\star)}
\ge -\frac{\Delta_\star}{r_{\max}}.
\]
Since $r_{i_m(s^\star)}\le r_{\max}$, multiplying
both sides by $r_{i_m(s^\star)}$ gives
\[
r_{i_m(s^\star)}
\bigl(\hat\psi_{i_m(s^\star)}(t)
-\psi_{i_m(s^\star)}\bigr)
\ge -r_{i_m(s^\star)}\cdot
\frac{\Delta_\star}{r_{\max}}
\ge -\Delta_\star
\]
for every $m$. Averaging over $m=1,\ldots,M$:
\begin{align*}
\hat g_{\mathrm{shared}}(s^\star,t)-g(s^\star)
&= \frac{1}{M}\sum_{m=1}^{M}
r_{i_m(s^\star)}
\bigl(\hat\psi_{i_m(s^\star)}(t)
-\psi_{i_m(s^\star)}\bigr)\\
&\ge \frac{1}{M}\sum_{m=1}^{M}
(-\Delta_\star)
= -\Delta_\star.
\end{align*}
Hence
\[
\hat g_{\mathrm{shared}}(s^\star,t)
\ge g(s^\star)-\Delta_\star
=\tau_r,
\]
contradicting
\[
\hat g_{\mathrm{shared}}(s^\star,t)<\tau_r.
\]
Therefore at least one $m\in[M]$ satisfies
\[
\hat\psi_{i_m(s^\star)}(t)
-\psi_{i_m(s^\star)}
<-\Delta_\star/r_{\max}.
\]

\subsection{Proof of Corollary \ref{cor:cts_regret}}

\rev{Consider the infinite-horizon execution and let
$A_i:=\{\text{CTS round }i\text{ starts}\}$, as in
Proposition~\ref{prop:round_termination}.
Conditional on $A_i$, the reset priors, fresh
algorithmic randomness, and future i.i.d.\ feedback
give a fresh CTS instance of length $2^i$.}

\rev{Under realizability, an optimal play incurs zero
satisficing regret, while every other play incurs
at most $\tau_r$. Hence, whenever $\mathbb P(A_i)>0$,
Lemma~\ref{lem:subopt_direct} gives}
\[
\rev{
\mathbb E\!\left[
\sum_{t\in\text{round }i}\Delta_{\mathrm{sat}}(S_t)
\,\middle|\, A_i
\right]
\le
\tau_r\bigl(\hat C_1 i\log 2+\hat C_0\bigr).
}
\]

\rev{The CTS regret up to time $T$ is bounded by
the total CTS regret over the infinite-horizon
execution. Since regret is nonnegative and a round
contributes zero if it never starts, summing the
expected contributions of all CTS rounds yields}
\begin{align*}
&\rev{
\mathbb E\!\left[
\sum_{t\in\mathcal T_{\mathrm{CTS}}}
\Delta_{\mathrm{sat}}(S_t)
\right]
}\\
&\rev{\le
\tau_r\sum_{i=1}^{\infty}
\bigl(\hat C_1 i\log 2+\hat C_0\bigr)
\mathbb P(A_i).
}
\end{align*}

\rev{Using $\mathbb P(A_1)\le1$ and applying
Proposition~\ref{prop:round_termination} with
$\mathbb P(A_i)\le1$, we obtain}
\begin{align*}
&\rev{
\mathbb E\!\left[
\sum_{t\in\mathcal T_{\mathrm{CTS}}}
\Delta_{\mathrm{sat}}(S_t)
\right]
}\\
&\rev{\le
\tau_r\Bigg[
\hat C_1\log 2+\hat C_0
}\\
&\rev{\quad+
\sum_{i=1}^{\infty}
\bigl(\hat C_1(i+1)\log 2+\hat C_0\bigr)
}\\
&\rev{\qquad\times\left(
\frac{2(\hat C_1 i\log 2+\hat C_0)}{2^i}
\right.
}\\
&\rev{\qquad\qquad\left.
+
\frac{M\exp(-2^i\Delta_\star^2/r_{\max}^2)}
{1-\exp(-2\Delta_\star^2/r_{\max}^2)}
\right)\Bigg]
}\\
&\rev{=R_{\mathrm{CTS}}.
}
\end{align*}

\rev{The summands involving $2^{-i}$ are
$O(i^2 2^{-i})$ and hence summable.
For the remaining terms, set
$c:=\Delta_\star^2/r_{\max}^2>0$.
Since $2^i\ge i$, we have $e^{-c2^i}\le e^{-ci}$,
and $\sum_{i=1}^{\infty}(i+1)e^{-ci}<\infty$.
The denominator $1-e^{-2c}$ is a positive constant.
Therefore, $R_{\mathrm{CTS}}$ is finite and
independent of~$T$.}

\section{Standard Regret Analysis for Non-Realizable Target}
\label{sec:nonrealizable}

In this section we analyze SAT-CTS when the satisficing threshold is non-realizable. In this case, no assignment can achieve the target average throughput. We show that the expected number of rounds in which the LCB or MEAN gate fires is finite. We also show that the expected number of rounds in which SAT-CTS does not execute a CTS step is finite. The remaining regret is therefore governed by the sum of the regret contributions of committed CTS rounds. Since each committed CTS round is restarted with fresh priors, the total standard regret is obtained by summing the regret over all CTS rounds, which yields an $O((\log T)^2)$ bound, similar to the non-realizable regret bound in \cite{feng2025satisficing}.

We use the same notation as in Section~\ref{sec:main_guarantees}. Recall that $g(s)$, $g^\star$, and $s^\star$ denote the average expected throughput, optimal throughput, and unique optimal assignment, respectively.
Recall $\Delta^{\mathrm{nr}}(s)$ from Section~\ref{sec:main_guarantees}. Under Assumption~\ref{ass:nonrealizable}, we have
\[
\Delta^{\mathrm{nr}}(s)\ge \Delta_\star^{\mathrm{nr}} > 0
\qquad \text{for all } s\in\mathcal S.
\]

\subsection{LCB and MEAN Gates under Non-Realizability}

We first show that on the good event, the LCB gate cannot fire, while the MEAN gate can fire only finitely many times in expectation.

\begin{lemma}[LCB gate cannot fire on the good event]
\label{lem:lcb_nonrealizable}
For all $t\ge 1$, on $\mathcal G_t$, the LCB gate does not fire. Equivalently,
\[
\mathbf 1\{\mathrm{Avg}(\mathrm{LCB},S_L)\ge \tau_r,\ \mathcal G_t\}=0
\qquad \text{a.s.}
\]
\end{lemma}

\begin{proof}
Let $S_L$ be the assignment returned by the LCB oracle at time $t \geq 1$. By Lemma~\ref{lem:lcb_conservative}, on $\mathcal G_t$ we have
\[
g(S_L)\ge \frac{1}{M}\sum_{m=1}^{M}\mathrm{LCB}_{i_m(S_L)}(t)
=\mathrm{Avg}(\mathrm{LCB},S_L).
\]
Hence, if the LCB gate fired on $\mathcal G_t$, then
\[
g(S_L)\ge \mathrm{Avg}(\mathrm{LCB},S_L)\ge \tau_r.
\]
This contradicts Assumption~\ref{ass:nonrealizable}, since $g(S_L)\le g^\star<\tau_r$. Therefore the LCB gate cannot fire on $\mathcal G_t$.
\end{proof}
\begin{definition}[MEAN-gate firing count]
\label{def:mean_count_nonrealizable}
Let $\mathcal T_{\mathrm{gate}}(T)\subseteq[T]$ be the set of rounds
at which the gates are evaluated, namely, rounds outside committed CTS
phases. For $t\in\mathcal T_{\mathrm{gate}}(T)$, let $S_L$ and $S_M$
denote the assignments returned by the LCB and MEAN oracles,
respectively. Define
\begin{align*}
N_{\mathrm{MEAN}}(T)
&:=
\sum_{t\in\mathcal T_{\mathrm{gate}}(T)}
\mathbf 1\left\{
\mathrm{Avg}(\mathrm{LCB},S_L)<\tau_r, \right.
\notag \\
& \qquad \qquad \qquad \left. \hat g_{\mathrm{shared}}(S_M,t)\ge\tau_r
\right\}.
\end{align*}
On every counted round, the MEAN gate fires and the algorithm plays
$S_t=S_M$.
\end{definition}

\begin{lemma}[A MEAN-gate firing implies a single-arm upward deviation]
\label{lem:mean_nonrealizable_single}
For every $t\in\mathcal T_{\mathrm{gate}}(T)$,
\begin{align*}
& \left\{
\mathrm{Avg}(\mathrm{LCB},S_L)<\tau_r,\;
\hat g_{\mathrm{shared}}(S_M,t)\ge\tau_r
\right\} \\
\qquad &\subseteq
\bigcup_{m=1}^{M}
E_{i_m(S_M),t}\!\left(
\frac{\Delta^{\mathrm{nr}}(S_M)}
{r_{i_m(S_M)}}
\right).
\end{align*}
On the event on the left, the MEAN gate fires and
$S_t=S_M$.
\end{lemma}

\begin{proof}
If the MEAN gate fires at time $t\ge1$, then
\[
\hat g_{\mathrm{shared}}(S_M,t)\ge \tau_r.
\]
Since $g(S_M)<\tau_r$ under Assumption~\ref{ass:nonrealizable}, it follows that
\[
\hat g_{\mathrm{shared}}(S_M,t)-g(S_M)\ge \tau_r-g(S_M)=\Delta^{\mathrm{nr}}(S_M).
\]
Write
\begin{align*}
& \hat g_{\mathrm{shared}}(S_M,t)-g(S_M) \\
&\qquad =
\frac{1}{M}\sum_{m=1}^{M}
r_{i_m(S_M)}
\bigl(\hat\psi_{i_m(S_M)}(t)-\psi_{i_m(S_M)}\bigr).
\end{align*}
If for all $m\in[M]$ we had
\[
\hat\psi_{i_m(S_M)}(t)-\psi_{i_m(S_M)}
<
\frac{\Delta^{\mathrm{nr}}(S_M)}{r_{i_m(S_M)}},
\]
then
\[
\hat g_{\mathrm{shared}}(S_M,t)-g(S_M)
<
\frac{1}{M}\sum_{m=1}^{M}\Delta^{\mathrm{nr}}(S_M)
=
\Delta^{\mathrm{nr}}(S_M),
\]
which is a contradiction. Hence there exists at least one $m\in[M]$ such that
\[
\hat\psi_{i_m(S_M)}(t)-\psi_{i_m(S_M)}
\ge \frac{\Delta^{\mathrm{nr}}(S_M)}{r_{i_m(S_M)}}.
\]
Since the MEAN gate fired, $S_t=S_M$, so $i_m(S_M)\in S_t$ and the
membership requirement in \eqref{eqn:overestimate_def} holds. This proves the claim.
\end{proof}

\begin{lemma}[Finite expected number of MEAN-gate firings]
\label{lem:mean_nonrealizable_finite}
For all $T\ge 1$,
\[
\mathbb E[N_{\mathrm{MEAN}}(T)]
\le
\sum_{i\in\mathcal A}\frac{r_i^2}{2(\Delta_\star^{\mathrm{nr}})^2}.
\]
\end{lemma}

\begin{proof}
By Lemma~\ref{lem:mean_nonrealizable_single}, every MEAN-gate firing at time $t\ge1$ implies that for at least one constituent arm $i_m(S_M)$,
\[
\hat\psi_{i_m(S_M)}(t)-\psi_{i_m(S_M)}
\ge \frac{\Delta^{\mathrm{nr}}(S_M)}{r_{i_m(S_M)}}.
\]
Since $\Delta^{\mathrm{nr}}(S_M)\ge \Delta_\star^{\mathrm{nr}}$, it follows that
\[
N_{\mathrm{MEAN}}(T)
\le
\sum_{i\in\mathcal A}\sum_{t=1}^{T}
\mathbf 1\left\{
E_{i,t}\!\left(\frac{\Delta_\star^{\mathrm{nr}}}{r_i}\right)
\right\}.
\]
Taking expectations and applying Lemma~\ref{lem:one_arm_geo} with $\varepsilon=\Delta_\star^{\mathrm{nr}}/r_i$ gives
\[
\mathbb E[N_{\mathrm{MEAN}}(T)]
\le
\sum_{i\in\mathcal A}
\frac{1}{2(\Delta_\star^{\mathrm{nr}}/r_i)^2}
=
\sum_{i\in\mathcal A}\frac{r_i^2}{2(\Delta_\star^{\mathrm{nr}})^2}.
\]
\end{proof}

\subsection{Transient Contribution Outside CTS Rounds}

We now bound the number of rounds in which SAT-CTS does not execute a CTS step.

\begin{definition}[Non-CTS rounds]
\label{def:noncts_rounds_nonrealizable}
Let
\[
\mathcal T_{\mathrm{nonCTS}}(T)
:=
\{t\in[T]: \text{no CTS step is executed at time } t\}.
\]
These are the rounds in which SAT-CTS does not execute a CTS step.
\end{definition}

\begin{lemma}[Finite expected number of non-CTS rounds]
\label{lem:finite_noncts_nonrealizable}
For all $T\ge 1$,
\[
\mathbb E\bigl[|\mathcal T_{\mathrm{nonCTS}}(T)|\bigr]
\le
\frac{\pi^2}{6}|\mathcal A|
+
\sum_{i\in\mathcal A}\frac{r_i^2}{2(\Delta_\star^{\mathrm{nr}})^2}.
\]
\end{lemma}

\begin{proof}
For $t\ge1$, on $\mathcal G_t$, the LCB gate cannot fire by Lemma~\ref{lem:lcb_nonrealizable}. Therefore, if SAT-CTS does not execute a CTS step at time $t\ge1$, then either $\mathcal G_t^c$ occurs or the MEAN gate fires. Hence
\[
|\mathcal T_{\mathrm{nonCTS}}(T)|
\le
\sum_{t=1}^{T}\mathbf 1\{\mathcal G_t^c\}
+
N_{\mathrm{MEAN}}(T).
\]
Taking expectations, using
\[
\sum_{t\ge1}\mathbb P(\mathcal G_t^c)
\le \frac{\pi^2}{6}|\mathcal A|
\]
from Lemma~\ref{lem:good_event_cost}, and applying Lemma~\ref{lem:mean_nonrealizable_finite}, we get
\[
\mathbb E\bigl[|\mathcal T_{\mathrm{nonCTS}}(T)|\bigr]
\le
\frac{\pi^2}{\rev{6}}|\mathcal A|
+
\sum_{i\in\mathcal A}\frac{r_i^2}{2(\Delta_\star^{\mathrm{nr}})^2}.
\]
\end{proof}

\subsection{Proof of Theorem~\ref{thm:preview_nonrealizable}}

Recall that the algorithm executes CTS in committed phases of lengths $2^1,2^2,\ldots,2^j,\ldots$, where each committed phases is restarted with fresh Beta$(1,1)$ priors. Therefore, the regret accumulated during CTS rounds is obtained by summing the regret bounds of fresh CTS instances over the committed phases.
Let
$
J(T):=\left\lceil \log_2\bigl(\max\{1,T+2\}\bigr) \right\rceil
$.
Up to time $T$, there are at most $J(T)$ committed CTS phases.

\begin{lemma}[CTS-round regret contribution]
\label{lem:cts_round_sum_nonrealizable}
The expected standard regret accumulated over CTS rounds up to time $T$ is at most
\[
\sum_{j=1}^{J(T)}
\Delta_{\max}\bigl(\hat C_1\, j\log 2+\hat C_0\bigr),
\]
where $\hat C_1,\hat C_0$ are from Lemma~\ref{lem:subopt_direct} and $\Delta_{\max}$ is from Section~\ref{sec:main_guarantees}.
\end{lemma}

\begin{proof}
The $j$th committed CTS round runs a fresh CTS instance for at most $2^j$ steps. By Lemma~\ref{lem:subopt_direct},
\[
\mathbb E[B_{\mathrm{sub}}(2^j)]
\le
\hat C_1 \log(2^j)+\hat C_0
=
\hat C_1\,j\log 2+\hat C_0.
\]
Since each suboptimal play incurs at most $\Delta_{\max}$ standard regret, the expected standard regret of the $j$th round is at most
\[
\Delta_{\max}\,\mathbb E[B_{\mathrm{sub}}(2^j)]
\le
\Delta_{\max}\bigl(\hat C_1\,j\log 2+\hat C_0\bigr).
\]
Summing over all committed CTS rounds completed by time $T$ yields the result.
\end{proof}

The theorem given below is the full version of
Theorem~\ref{thm:preview_nonrealizable}, which specifies,
up to the universal constant in Lemma~\ref{lem:subopt_direct},
the constants appearing in the $O\!\big((\log T)^2\big)$ term
in the standard regret bound.
\begin{theorem}[Standard regret under non-realizability]
\label{thm:nonrealizable_main}
Under Assumptions~\ref{ass:nonrealizable},~\ref{ass:unique_optimal}, and \ref{ass:independent_outcomes}  the expected standard regret of SAT-CTS satisfies
\[
\mathbb E[R^{\mathrm{std}}(T)]
\le
R_{\mathrm{trans}}^{\mathrm{nr}}
+
\sum_{j=1}^{J(T)}
\Delta_{\max}\bigl(\hat C_1\, j\log 2+\hat C_0\bigr),
\]
where
\[
R_{\mathrm{trans}}^{\mathrm{nr}}
:=
\Delta_{\max}
\left(
\frac{\pi^2}{\rev{6}}|\mathcal A|
+
\sum_{i\in\mathcal A}\frac{r_i^2}{2(\Delta_\star^{\mathrm{nr}})^2}
\right),
\]
Consequently,
\[
\mathbb E[R^{\mathrm{std}}(T)]
=
O\!\big((\log T)^2\big).
\]
\end{theorem}

\begin{proof}
Decompose the standard regret into regret incurred on CTS rounds and regret incurred on non-CTS rounds:
\begin{align*}
R^{\mathrm{std}}(T)
&=
\sum_{t\in\mathcal T_{\mathrm{CTS}}(T)}
\bigl(g(s^\star)-g(S_t)\bigr)
\\
&\qquad +
\sum_{t\in\mathcal T_{\mathrm{nonCTS}}(T)}
\bigl(g(s^\star)-g(S_t)\bigr),
\end{align*}
where
\[
\mathcal T_{\mathrm{CTS}}(T):=[T]\setminus \mathcal T_{\mathrm{nonCTS}}(T).
\]

For the non-CTS rounds, each round contributes at most $\Delta_{\max}$. Hence
\[
\mathbb E\!\left[
\sum_{t\in\mathcal T_{\mathrm{nonCTS}}(T)}
\bigl(g(s^\star)-g(S_t)\bigr)
\right]
\le
\Delta_{\max}\,
\mathbb E\bigl[|\mathcal T_{\mathrm{nonCTS}}(T)|\bigr].
\]
Applying Lemma~\ref{lem:finite_noncts_nonrealizable}, we obtain
\[
\mathbb E\!\left[
\sum_{t\in\mathcal T_{\mathrm{nonCTS}}(T)}
\bigl(g(s^\star)-g(S_t)\bigr)
\right]
\le
R_{\mathrm{trans}}^{\mathrm{nr}}.
\]
For the CTS rounds, Lemma~\ref{lem:cts_round_sum_nonrealizable} yields
\[
\mathbb E\!\left[
\sum_{t\in\mathcal T_{\mathrm{CTS}}(T)}
\bigl(g(s^\star)-g(S_t)\bigr)
\right]
\le
\sum_{j=1}^{J(T)}
\Delta_{\max}\bigl(\hat C_1\, j\log 2+\hat C_0\bigr).
\]
Combining the two bounds gives
\[
\mathbb E[R^{\mathrm{std}}(T)]
\le
R_{\mathrm{trans}}^{\mathrm{nr}}
+
\sum_{j=1}^{J(T)}
\Delta_{\max}\bigl(\hat C_1\, j\log 2+\hat C_0\bigr).
\]
Finally, since $\sum_{j=1}^{J(T)} j = O(J(T)^2)$, $\sum_{j=1}^{J(T)} 1 = O(J(T))$, and $J(T)=\lceil \log_2(\max\{1,T+ 2\}) \rceil = O(\log T)$, it follows that
\begin{align*}
\mathbb E[R^{\mathrm{std}}(T)]
&= O\!\big((\log T)^2\big).
\end{align*}
\end{proof}

\clearpage
\raggedbottom
\section{Pseudocodes of CUCB and CTS}

In this section, we represent the full pseudocodes of the competitor algorithms we used in Section \ref{sec:experiments}.

\begin{algorithm}[H]
\caption{CUCB: Combinatorial Upper Confidence Bound}
\label{alg:cucb}
\begin{algorithmic}[1]
\REQUIRE $M$ UEs; beam set $\mathcal{K}=[B]\times[K]$;
  \rev{RF-chain counts $\{N_{\mathrm{RF},b}\}_{b=1}^{B}$;}
  rate set $\mathcal{R}=\{r_1<\cdots<r_R\}$;
\STATE \textbf{Initialize:}
\STATE \quad $n_{m,(b,k),r}\!\leftarrow\!0,\;\; \hat{\psi}_{m,(b,k),r}\!\leftarrow\!0$ for all $m\in[M],\ (b,k)\in\mathcal K,\ r\in\mathcal R$ \COMMENT{Counts and success-rate estimates}
\FOR{$t=1$ to $T$}
    \STATE \textbf{Step 1: UCB index computation on success probability}
    \FOR{each $m\in[M],\ (b,k)\in\mathcal K,\ r\in\mathcal R$}
        \STATE \quad \textbf{if } $n_{m,(b,k),r}=0$ \textbf{ then } $\mathrm{UCB}_{m,(b,k),r}\leftarrow +\infty$ \textbf{ else }
        \STATE \quad $\mathrm{UCB}_{m,(b,k),r} \gets
\min\!\left\{1,\,
\widehat{\psi}_{m,(b,k),r}
+\sqrt{\frac{3\log(\max\{2,t\})}{2n_{m,(b,k),r}}}
\right\}$
        \STATE \quad $\theta_{m,(b,k),r} \leftarrow r \cdot \mathrm{UCB}_{m,(b,k),r}$ 
    \ENDFOR
    \STATE
    \STATE \textbf{Step 2: Construct super-arm via optimal assignment}
\STATE $\mathcal{S} \!=\! \bigl\{\{(\pi_1,\rho_1),\ldots,(\pi_M,\rho_M)\}:$
\quad $\pi_m=(b_m,k_m)\!\in\!\mathcal K,\;
\rho_m\!\in\!\mathcal R,$
\quad $\left|\{m\in[M]:b_m=b\}\right|
\le N_{\mathrm{RF},b}\;\forall\,b\!\in\![B]\bigr\}$
    \STATE \quad $\mathcal{A}_t \leftarrow \arg\max_{s \in \mathcal{S}} \sum_{m=1}^{M} \theta_{m,\pi_m(s),\rho_m(s)}$
    \STATE
    \STATE \textbf{Step 3: Assign super-arm and update estimates (ACK/NACK)}
    \STATE \quad Assign $\mathcal{A}_t=\{(\pi_{1,t},\rho_{1,t}),\ldots,(\pi_{M,t},\rho_{M,t})\}$ with $\pi_{m,t}=(b_{m,t},k_{m,t})$
    \STATE \quad Observe $x_{m,t}\in\{0,1\}$ for $m=1,\ldots,M$
    \FOR{each user $m\in[M]$}
        \STATE \quad $n_{m,\pi_{m,t},\rho_{m,t}} \leftarrow n_{m,\pi_{m,t},\rho_{m,t}} + 1$
        \STATE \quad $\hat{\psi}_{m,\pi_{m,t},\rho_{m,t}} \leftarrow \hat{\psi}_{m,\pi_{m,t},\rho_{m,t}} + \dfrac{x_{m,t} - \hat{\psi}_{m,\pi_{m,t},\rho_{m,t}}}{\,n_{m,\pi_{m,t},\rho_{m,t}}\,}$
    \ENDFOR
\ENDFOR
\end{algorithmic}
\end{algorithm}

\begin{algorithm}[H]
\caption{CTS: Combinatorial Thompson Sampling}
\label{alg:cts}
\begin{algorithmic}[1]
\REQUIRE $M$ users, beam set $\mathcal{K}=[B]\times[K]$, rate set $\mathcal{R}=\{r_1<\cdots<r_R\}$, time horizon $T$
\rev{RF-chain counts $\{N_{\mathrm{RF},b}\}_{b=1}^{B}$;}
\STATE \textbf{Initialize:}
\STATE \quad $A_{m,(b,k),r}\!\leftarrow\!1,\;\; B_{m,(b,k),r}\!\leftarrow\!1$ for all $m\in[M],\ (b,k)\in\mathcal K,\ r\in\mathcal R$ \COMMENT{Beta priors on $\psi$}
\FOR{$t=1$ to $T$}
    \STATE \textbf{Step 1: Thompson sampling from posterior}
    \FOR{each $m\in[M],\ (b,k)\in\mathcal K,\ r\in\mathcal R$}
        \STATE \quad $\tilde\psi_{m,(b,k),r} \sim \mathrm{Beta}(A_{m,(b,k),r}, B_{m,(b,k),r})$
        \STATE \quad $\theta_{m,(b,k),r} \leftarrow r\cdot \tilde\psi_{m,(b,k),r}$ \COMMENT{expected throughput sample}
    \ENDFOR
    \STATE
    \STATE \textbf{Step 2: Construct super-arm via optimal assignment}
    \STATE $\mathcal{S} \!=\! \bigl\{\{(\pi_1,\rho_1),\ldots,(\pi_M,\rho_M)\}:$
    \quad $\pi_m=(b_m,k_m)\!\in\!\mathcal K,\;
    \rho_m\!\in\!\mathcal R,$
    \quad $\left|\{m\in[M]:b_m=b\}\right|
    \le N_{\mathrm{RF},b}\;\forall\,b\!\in\![B]\bigr\}$
    \STATE \quad $\mathcal{A}_t \leftarrow \arg\max_{s \in \mathcal{S}} \sum_{m=1}^{M} \theta_{m,\pi_m(s),\rho_m(s)}$
    \STATE
    \STATE \textbf{Step 3: Assign super-arm and update posteriors (ACK/NACK)}
    \STATE \quad Assign $\mathcal{A}_t=\{(\pi_{1,t},\rho_{1,t}),\ldots,(\pi_{M,t},\rho_{M,t})\}$ with $\pi_{m,t}=(b_{m,t},k_{m,t})$
    \STATE \quad Observe $x_{m,t}\in\{0,1\}$ for $m=1,\ldots,M$
    \FOR{each user $m\in[M]$}
        \STATE \quad $A_{m,\pi_{m,t},\rho_{m,t}} \leftarrow A_{m,\pi_{m,t},\rho_{m,t}} + x_{m,t}$
        \STATE \quad $B_{m,\pi_{m,t},\rho_{m,t}} \leftarrow B_{m,\pi_{m,t},\rho_{m,t}} + (1 - x_{m,t})$
    \ENDFOR
\ENDFOR
\end{algorithmic}
\end{algorithm}

\bibliography{references}
\bibliographystyle{IEEEtran}

\end{document}